\documentclass[final]{alt2023}

\newcommand{\titre}{On Best-Arm Identification with a Fixed Budget \\ in Non-Parametric Multi-Armed Bandits}
\newcommand{\titrecourt}{Fixed Budget Best-Arm Identification in Non-Parametric Multi-Armed Bandits}
\title[\titrecourt]{\titre}

\usepackage{times}
\usepackage{color}
\usepackage{thmtools}
\usepackage{thm-restate}

\usepackage[most]{tcolorbox}
\newtcolorbox{nbox}[1][]{
  enhanced,
  fonttitle=\scshape,
  #1
}

\altauthor{\Name{Antoine Barrier} \Email{antoine.barrier@ens-lyon.fr}\\
 \addr ENS de Lyon, UMPA UMR 5669, 46 allée d’Italie, 69364 Lyon Cedex 07, France \\
 \addr Université Paris-Saclay, CNRS, Laboratoire de mathématiques d’Orsay, 91405, Orsay, France
 \AND
 \Name{Aur{\'e}lien Garivier} \Email{aurelien.garivier@ens-lyon.fr}\\
 \addr ENS de Lyon, UMPA UMR 5669 et LIP UMR 5668, 46 allée d’Italie, 69364 Lyon Cedex 07, France
 \AND
 \Name{Gilles Stoltz} \Email{gilles.stoltz@universite-paris-saclay.fr}\\
 \addr Université Paris-Saclay, CNRS, Laboratoire de mathématiques d’Orsay, 91405, Orsay, France
}

\renewcommand{\leq}{\leqslant}
\renewcommand{\geq}{\geqslant}
\renewcommand{\P}{\mathbb{P}}
\newcommand{\E}{\mathbb{E}}
\newcommand{\R}{\mathbb{R}}

\newcommand{\unu}{\underline{\nu}}
\newcommand{\ula}{\underline{\lambda}}
\newcommand{\cA}{\mathcal{A}}
\newcommand{\cB}{\mathcal{B}}
\newcommand{\cD}{\mathcal{D}}

\newcommand{\cK}{\mathcal{K}}
\newcommand{\cP}{\mathcal{P}}

\newcommand{\cN}{\mathcal{N}}
\newcommand{\eqdef}{\stackrel{\mbox{\tiny\rm def}}{=}}

\newcommand{\base}{\mbox{\tiny base}}
\newcommand{\ind}[1]{\mathbb{I}_{\{ {#1} \}}}
\newcommand{\KL}{\mathrm{KL}}
\newcommand{\Ed}{\mathrm{E}}
\newcommand{\lriT}{\underset{T \to +\infty}{\longrightarrow}}
\newcommand{\Ber}{\mathrm{Ber}}

\newcommand{\ov}[1]{\overline{#1}}
\newcommand{\argmin}{\mathop{\mathrm{argmin}}}

\DeclareMathOperator{\Supp}{Supp}
\DeclareMathOperator{\Alt}{Alt}

\newcommand{\cDall}{\cP[0,1]}
\newcommand{\cDexp}{\cD_{\text{exp}}}
\newcommand{\cDber}{\mathcal{B}_{[1/4,\,3/4]}}
\newcommand{\cDberp}{\mathcal{B}_{[p,\,1-p]}}
\renewcommand{\d}{\mathrm{d}}
\renewcommand{\epsilon}{\varepsilon}
\newcommand{\kl}{\mathrm{kl}}
\newcommand{\e}{\mathrm{e}}
\newcommand{\ol}{\overline}
\newcommand{\cG}{\mathcal{G}}
\newcommand{\cI}{\mathcal{I}}
\newcommand{\cM}{\mathcal{M}}
\newcommand{\idunit}{\mathrm{id}_{[0,1]}}
\newcommand{\idR}{\mathrm{id}_{\R}}

\newcommand{\Linfstr}{\mathcal{L}_{\inf}^{<}}
\newcommand{\Linfeq}{\mathcal{L}_{\inf}^{\leq}}
\newcommand{\Lsupstr}{\mathcal{L}_{\inf}^{>}}
\newcommand{\Lsupeq}{\mathcal{L}_{\inf}^{\geq}}
\newcommand{\Ksupstr}{\mathcal{K}_{\inf}^{>}}
\newcommand{\chern}{\mathcal{L}}
\newcommand{\bln}{\mathop{\ov{\ln}}}

\newcommand{\sk}{\textbf{sketch}}

\begin{document}

\maketitle

\begin{abstract}%
We lay the foundations of a non-parametric theory of best-arm identification in multi-armed bandits with a fixed budget~$T$.
We consider general, possibly non-parametric, models $\cD$
for distributions over the arms; an overarching example is the model $\cD = \cDall$
of all probability distributions over $[0,1]$.
We propose upper bounds on the average log-probability of misidentifying the optimal arm based on information-theoretic quantities that we name $\Linfstr(\,\cdot\,,\nu)$ and $\Lsupstr(\,\cdot\,,\nu)$ and that correspond to
infima over Kullback-Leibler divergences between some distributions in $\cD$ and a given distribution $\nu$.
This is made possible by a refined analysis of the successive-rejects strategy of \citet{ABM10}.
We finally provide lower bounds on the same average log-probability, also in terms of the same new information-theoretic quantities; these lower bounds are larger when the (natural) assumptions on the considered strategies are stronger.
All these new upper and lower bounds generalize existing bounds based, e.g., on gaps between distributions.
\end{abstract}


\section{Introduction and brief literature review}
\label{sec:intro}

We consider a class $\cD$ of distributions over $\R$ with finite first moments,
which we refer to as the model $\cD$.
A $K$--armed bandit problem in $\cD$ is a $K$--tuple
$\unu = (\nu_1,\ldots,\nu_K)$ of distributions in $\cD$.
We denote by $(\mu_1,\ldots,\mu_K)$ the $K$--tuple of their expectations.
An agent sequentially interacts with $\unu$:
at each step $t \geq 1$, she selects an arm $A_t$
and receives a reward $Y_t$ drawn from the distribution $\nu_{A_t}$. This is the only feedback that she obtains.

While regret minimization has been vastly studied (see \citealp{LS20}), another relevant objective is \emph{best-arm identification}, that is, identifying the distribution with highest expectation.
In the fixed-confidence setting, this identification is performed  under the constraint that a given confidence level $1-\delta$ is respected, while minimizing the expected number of pulls of the arms (the expected sample complexity). This setting is fairly well understood
(see \citealp[Chapter~33]{LS20} for a review). A turning point in this literature was achieved by~\cite{GK16}, who
provided matching upper and lower bounds on the expected number of pulls of the arms in the case of canonical one-parameter exponential families.
Since then, improvements have been made in several directions, including for example non-asymptotic bounds (\citealp{DKM19}) and the problem of $\varepsilon$--best-arm identification (\citealp{GK21}).
The first generalization to non-parametric models in this fixed-confidence setting
was achieved by~\citet{TopTwoNP}, who worked in a concurrent and independent manner
from us. Their upper and lower bounds differ by a multiplicative factor of~2 (only).
\medskip

\noindent\textbf{Best-arm identification with a fixed budget.}~~The \emph{fixed-budget setting} is much less understood
in our opinion. Therein, the total number $T$ of pulls of the arms is fixed.
After these $T$ pulls, a strategy must issue a recommendation $I_T$.
Assuming that $\unu$ contains a unique optimal distribution $\nu^\star$ of index $a^\star(\unu)$,
one aims at minimizing $\P\bigl( I_T \neq a^\star(\unu) \bigr)$.
We are interested in (upper and lower) bounds that hold for all problems $\unu$ in $\cD$,
possibly under the restriction that they only contain a unique optimal arm.
It may be straightforwardly seen that the probability of error can decay exponentially fast---for instance,
by uniformly exploring the arms (pulling each of them about $T/K$ times) and recommending the one with the largest empirical average.
This is why the literature (see, for instance, \citealp{ABM10} and \citealp[Chapter~33]{LS20})
focuses on upper and lower bound functions
$\ell \leq U < 0$ of the typical form:
\emph{for all bandit problems} $\unu$ in $\cD$, with a unique optimal arm,
\begin{multline*}
\ell(\unu) \leq \liminf_{T \to +\infty} \frac{1}{T} \ln \P\bigl( I_T \neq a^\star(\unu) \bigr) \leq \limsup_{T \to +\infty} \frac{1}{T} \ln \P\bigl( I_T \neq a^\star(\unu) \bigr) \leq U(\unu) < 0\,, \\
\mbox{or, put differently,} \qquad\quad
\exp \Bigl( \ell(\unu) \, T \bigl( 1 + o(1) \bigr) \Bigr)
\leq \P\bigl( I_T \neq a^\star(\unu) \bigr)
\leq \exp \Bigl( U(\unu) \, T \bigl( 1 + o(1) \bigr) \Bigr)\,.
\end{multline*}
This problem is generally considered more difficult than the fixed-confidence setting
(see, e.g., \citealp[Chapter~33]{LS20} and \citealp[Section~6]{TopTwoNP}),
and even for parametric models like canonical one-parameter exponential models, no strategy with matching
upper and lower bounds (i.e., no optimal strategy) is known so far.

\paragraph{Earlier approaches.}
So far, four main approaches were considered
for the problem of best-arm identification with a fixed budget.
\emph{First}, the early approach by~\cite{ABM10} relies on gaps:
we define the gap $\Delta_a$ of arm~$a$ as the difference
$\mu^\star - \mu_a$ between the largest expectation $\mu^\star$ in $\unu$
and the expectation of the distribution $\nu_a$.
They introduce a successive-rejects strategy and provide gap-based upper bounds
for sub-Gaussian models, based on Hoeffding's inequality. They however propose a lower
bound only in the case of a Bernoulli model, not for larger, non-parametric, models.
This lower bound was further discussed by~\citet{CL16}, in a minimax sense.
\emph{A second series of approaches} (see, e.g., \citealp{KCG16}) focused
on Gaussian bandits with fixed variances, but their results do not seem to be easily generalized
to other models as they rely on specific properties (even stronger than the symmetry of the Kullback-Leibler divergence, namely,
that in this model, the Kullback-Leibler divergence only depends on the gap between the expectations of the distributions).
\emph{A third approach}, led by \citet{Rus16,Rus20}, considered canonical
one-parameter exponential families,
but for a different target probability. Namely, a Bayesian setting is considered and
the quality of a strategy is measured as the posterior probability of identifying the best arm.
An optimal non-gap-based complexity is exhibited, together with optimal strategies matching this complexity. However, \cite{Kom22} argue that such an approach is specific to the Bayesian case and
is not suited to the frequentist case that we consider.
\emph{A fourth approach} is to focus on the case of $K=2$ arms, see, e.g., \cite{KCG16}.
The non-parametric bounds obtained therein do not enjoy any
obvious generalization to the case of $K \geq 3$ arms
beyond the one stated in Theorem~\ref{th:LB-conj-GK16}
and criticized in Section~\ref{sec:existing}
for only involving pairwise comparisons with the best arm.
By considering very specific models, \cite{KAINQ22} constructed a strategy that is optimal (only) in the regime where the gap between the $2$ arms is small---yet,
this gap-based approach does not, by nature, go in the direction of non-parametric bounds.

We will provide more details concerning some of these approaches
while presenting and discussing our main results, in Section~\ref{subsec:overview}; see also Appendix~\ref{app:literature}.

\paragraph{Content and outline of this article.}
We focus our attention on instance-dependent upper and lower bounds,
holding for all problems of general models $\cD$, including non-parametric models,
and valid for any number $K$ of arms. Put differently,
we target a high degree of generality. While admittedly not exhibiting
matching upper and lower bounds, we show that the same (new) information-theoretic
quantities $\Linfstr$ and $\Lsupstr$ are at stake in these upper and lower bounds. These
information-theoretic quantities are defined,
in Section~\ref{sec:overview},
as infima of Kullback-Leibler divergences and provide a quantification of the
difficulty of the identification in terms of the geometry of information of the problem.
We also present in this section an overview of our results,
which we carefully compare to existing bounds (restated therein,
occasionally with some improvements).
We state upper bounds in Section~\ref{sec:SR} and to do so, we
provide an improved analysis of the classical successive-rejects strategy, not relying on gaps through Hoeffding's lemma.
Section~\ref{sec:LB} exhibits several possible lower bounds, which are inversely larger
to the strength of the assumptions made on the strategies. These lower bounds
generalize known lower bounds in the literature, like the lower bound for Bernoulli
models by~\cite{ABM10}, but hold for arbitrary models. They share some similar
flavor with the lower bounds by~\cite{LaRo85} and~\cite{BuKa96} for the cumulative regret.

\section{Overview of the results and more extended literature review}
\label{sec:overview}

Before being able to actually provide a formal summary of our results,
we introduce new quantifications of the difficulty of a bandit
problem in terms of geometry of the information.

\subsection{The key new quantities: $\Linfstr$ and $\Linfeq$, as well as $\Lsupstr$ and $\Lsupeq$}
\label{sec:Linf}

In this article, we only consider models $\cD$ whose distributions all admit an expectation.
We denote by $\Ed(\zeta)$ the expectation of a distribution $\zeta \in \cD$.
For a distribution $\nu \in \cD$ and a real number $x \in \R$,
we then introduce
\begin{align*}
& \Linfstr(x,\nu) = \inf \bigl\{ \KL(\zeta,\nu) : \zeta \in \cD \ \ \mbox{s.t.}
\ \ \Ed(\zeta) < x \bigr\} \\
\mbox{and} \qquad &
\Linfeq(x,\nu) = \inf \bigl\{ \KL(\zeta,\nu) : \zeta \in \cD \ \ \mbox{s.t.}
\ \ \Ed(\zeta) \leq x \bigr\}\,,
\end{align*}
where $\KL$ denotes the Kullback-Leibler divergence and
with the usual convention that the infimum of an empty set equals~$+\infty$.
Symmetrically, by considering rather distributions $\zeta$ with expectations
larger than $x$, we define
\begin{align*}
& \Lsupstr(x,\nu) = \inf \bigl\{ \KL(\zeta,\nu) : \zeta \in \cD \ \ \mbox{s.t.}
\ \ \Ed(\zeta) > x \bigr\} \\
\mbox{and} \qquad &
\Lsupeq(x,\nu) = \inf \bigl\{ \KL(\zeta,\nu) : \zeta \in \cD \ \ \mbox{s.t.}
\ \ \Ed(\zeta) \geq x \bigr\}\,.
\end{align*}
We state some general properties on these quantities in Appendix~\ref{app:Linf}---among others, that $\Linfstr$ and $\Linfeq$,
as well as $\Lsupstr$ and $\Lsupeq$, are almost identical
for the model $\cDall$.
The same holds for canonical one-parameter exponential models, as discussed in Appendix~\ref{app:cDexp}.
Lower bounds will be typically expressed with $\Linfstr$ and $\Lsupstr$ quantities,
while upper bounds will rely on $\Linfeq$ and $\Lsupeq$ quantities.

\begin{remark}
\label{rk:Kinf}
The key quantities for the non-parametric study of best-arm identification with fixed confidence
by~\citet{TopTwoNP} are defined based on Kullback-Leibler divergences with arguments in reverse order,
namely,
\begin{align*}
& \cK_{\inf}^-(\nu,x) = \inf \bigl\{ \KL(\nu,\zeta) : \zeta \in \cD \ \ \mbox{\rm s.t.}
\ \ \Ed(\zeta) < x \bigr\} = \cK_{\inf}(\nu,x) \\
\mbox{and} \qquad &
\cK_{\inf}^+(\nu,x) = \inf \bigl\{ \KL(\nu,\zeta) : \zeta \in \cD \ \ \mbox{\rm s.t.}
\ \ \Ed(\zeta) > x \bigr\}\,,
\end{align*}
where the first quantity was referred to as simply $\cK_{\inf}(\nu,x)$
by \citet{honda_non-asymptotic_2015} in the regret-minimization literature
(see also Appendix~\ref{app:phi=L} and \citealp{KLUCBS}). Optimal
bounds for regret minimization only depend on $\cK_{\inf}(\nu,x)$.
\end{remark}

For best-arm identification with fixed budget,
the arguments in the $\KL$ are in reverse order compared to the fixed-confidence setting.
Except for very specific models (e.g., the model $\cD_{\sigma^2}$ of
Gaussian distributions with a fixed variance $\sigma^2 > 0$), the Kullback-Leibler divergence
is not symmetric, i.e., $\KL(\zeta,\nu)$ and $\KL(\nu,\zeta)$ differ in general.
Specific best-arm-identification results were obtained by~\citet{KCG16}
for the model $\cD_{\sigma^2}$, based on the Bretagnolle-Huber inequality (\citealp{BH79});
they indicate that the sum of the inverse squared gaps would be driving
both the lower bound and upper bound functions $\ell$ and $U$.
However, a close look at the proof reveals that they heavily rely on a property even stronger than the symmetry of
$\KL$ for this model: details and discussions on this matter are provided in Appendix~\ref{app:BH}.
In particular, generalizations beyond the Gaussian case appear to be infeasible.

\subsection{Overview of the results}	
\label{subsec:overview}

The paper provides new and more general (possibly non-parametric) bounds on the misidentification errors based on
the information-theoretic quantities introduced above.
In particular, we consider a version of Chernoff information defined, for $\nu,\nu'$ in $\cD$
with $\Ed(\nu') < \Ed(\nu)$, as
\begin{equation}
\label{eq:defchern}
\chern(\nu', \nu) = \inf_{x \in [\Ed(\nu'), \Ed(\nu)]} \Bigl\{ \Lsupeq(x, \nu') + \Linfeq(x, \nu) \Bigr\} \,.
\end{equation}
Given a bandit problem $\unu$ with a unique optimal distribution denoted by $\nu^\star$,
we may rank the arms $a$ in non-decreasing order of $\chern\bigl(\nu_{a},\nu^\star\bigr)$,
i.e., consider the permutation $\sigma$ such that
\begin{equation} \label{eq:ordering_by_L}
0 = \chern\bigl(\nu_{\sigma_1},\nu^\star)
< \chern\big(\nu_{\sigma_2},\nu^\star\bigr) \leq \ldots
\leq \chern\bigl(\nu_{\sigma_{K-1}},\nu^\star\bigr)
\leq \chern\bigl( \nu_{\sigma_K},\nu^\star\bigr)\,.
\end{equation}
\emph{Our first main result} (Corollary~\ref{cor:SR} together with Lemma~\ref{lem:duality})
considers models $\cD$
like $\cD = \cDall$, the set of all probability distributions over $[0,1]$,
or $\cD = \cDexp$, any canonical one-parameter exponential family. We study
the successive-rejects strategy, introduced by \citet{ABM10}, for which arms are rejected one by one at the end of phases of uniform exploration, and state that this strategy is such that
for all bandit problems $\unu$ in $\cD$ with a unique optimal arm,
\begin{equation} \label{eq:ov_boundSR}
\limsup_{T \to +\infty} \frac{1}{T} \ln \P \bigl( I_T \ne a^\star(\unu) \bigr)
\leq - \frac{1}{\bln K} \min_{2 \leq k \leq K} \frac{\chern\bigl( \nu_{\sigma_{k}},\nu^\star\bigr)}{k}\,,
\end{equation}
where $\bln K$ is defined in~\eqref{eq:reglengths-ABM10}
and is of order $\ln K$. The key for this result (Lemma~\ref{lem:devphistar}, of independent interest) is a grid-based
application of the Cram{\'e}r-Chernoff bound to control
$\P \bigl( \overline{X}_N \leq \overline{Y}_N \bigr)$, where $\overline{X}_N$ and $\overline{Y}_N$
are averages of two independent $N$--samples. This approach can be used to analyze similar algorithms, like sequential halving (\citealp{KKS13}).

The corresponding lower bounds are stated rather in terms of $\Linfstr$ and $\Lsupstr$ quantities,
but Appendix~\ref{app:Linf} explains why, except in a single pathological case,
$\chern(\nu', \nu)$ could be alternatively defined with $\Linfstr$ and $\Lsupstr$ instead of
$\Linfeq$ and $\Lsupeq$. We actually state several lower bounds in Section~\ref{sec:LB},
that are larger as the assumptions on the strategies considered are more restrictive; as usual, there is
a trade-off between the strength of a lower bound and its generality. However, all assumptions considered remain rather
mild and are satisfied by successive-rejects-type strategies:
for instance, Definition~\ref{eq:balanced} restricts the attention to strategies
such that for all bandit problems, the arm associated with the smallest expectation is pulled less than a fraction $1/K$ of the time.
Out of all lower bounds exhibited,
\emph{our second main result} (Theorem~\ref{th:LBmonotone}) holds, as indicated,
under mild assumptions on the model and sequences of strategies considered, and reads:
for all bandit problems $\unu$ with no two same expectations,
\begin{equation} \label{eq:ov_LB2}
\liminf_{T \to +\infty} \frac{1}{T} \ln \P_{\unu}\bigl(I_T \ne a^\star(\unu) \bigr)
\geq - \min_{2 \leq k \leq K} \,\, \inf_{x \in [ \mu_{(k)}, \mu_{(k-1)} )}
\biggl\{ \frac{\Lsupstr\bigl(x, \nu_{(k)}\bigr)}{k-1} + \frac{\Linfstr\bigl(x, \nu^\star\bigr)}{k} \biggr\} \,,
\end{equation}
where $\smash{\mu_{(1)} > \mu_{(2)} > \mu_{(3)} > \dots > \mu_{(K)}}$
and where $\smash{\nu_{(a)}}$ denotes the distribution with expectation $\mu_{(a)}$.
Here, we considered the notation $(k)$ for order statistics in reverse order.

This lower bound does not match the exhibited upper bound, as is further discussed in
Section~\ref{sec:discuss-opt}.
Still, we argue that quantities defined as infima over $x$ of
$\Lsupstr\bigl(x, \nu_{(k)}\bigr) + \Linfstr\bigl(x, \nu^\star\bigr)$ should measure how difficult a best-arm-identification
problem is under a fixed budget. \emph{This is the main insight of this article.}

\subsection{Re-derivation of existing bounds}
\label{sec:existing}

We now survey the most important existing bounds and re-derive them from our general bounds. These existing bounds
all hold only for sub-Gaussian models and for exponential models when $K \geq 3$,
while a non-parametric bound was only available in the case of $K=2$ arms.

To do so, we will sometimes consider the following weaker version of the lower bound~\eqref{eq:ov_LB2},
obtained by picking $x = \mu_{(k)}$:
\begin{equation} \label{eq:ov_LB2-bis}
\smash{\liminf_{T \to +\infty} \frac{1}{T} \ln \P_{\unu}\bigl(I_T \ne a^\star(\unu) \bigr)
\geq - \min_{2 \leq k \leq K} \frac{\Linfstr\bigl(\mu_{(k)}, \nu^\star\bigr)}{k} \,}.
\end{equation}

\paragraph*{Comparison to the gap-based approaches.}
\citet{ABM10} propose an analysis of the successive-rejects strategy based on Hoeffding's inequality,
stating that for all bandit problems in $\cDall$ with a unique optimal arm,
\begin{equation}
\label{eq:UB_ABM10}
\smash{\limsup_{T \to +\infty} \frac{1}{T} \ln \P_{\unu}\bigl(I_T \ne a^\star(\unu)\bigr)
\leq - \frac{1}{\bln K} \min_{2 \leq k \leq K} \frac{\Delta_{(k)}^2}{k}} \,, \vspace{.3cm}
\end{equation}
where we recall the definition of the gaps $\Delta_{(k)} = \mu^\star - \mu_{(k)}$.
This bound is a consequence of (Corollary~\ref{cor:SR}, a slightly more general form of) the bound~\eqref{eq:ov_boundSR}, given
Pinsker's inequality~\eqref{eq:Pinsker-Linf}:
\begin{equation}
\label{eq:chern-gaps}
\chern\bigl(\nu_{(k)},\nu^\star\bigr) \geq
\inf_{x \in [\mu_{(k)}, \mu^\star]} \Bigl\{ 2 \bigl( x - \mu_{(k)} \bigr)^2 + 2(x - \mu^\star)^2 \Bigr\} =
\bigl( \mu^\star - \mu_{(k)} \bigr)^2 = \Delta_{(k)}^2\,. \vspace{-.125cm}
\end{equation}
We remark that the bound~\eqref{eq:UB_ABM10} and the lower bound on $\chern\bigl(\nu_{(k)},\nu^\star\bigr)$ may actually
be extended to the model of $\sigma^2$--sub-Gaussian distributions, up to considering factors $1/(4\sigma^2)$.
We do not discuss the UCB-E algorithm of \citet{ABM10}, as
its performance and analysis crucially depend on a tuning parameter set
with some knowledge of the gaps.

\citet{ABM10} also propose a carefully constructed lower bound for the model $\cDberp = \bigl\{ \Ber(x) : x \in [p,1-p] \bigr\}$
of Bernoulli distributions $\Ber(x)$ with parameters $x$ in $[p,1-p]$ for some $p \in (0, 1/2)$. A key inequality in their proof follows from the
Kullback-Leibler -- $\chi^2$-divergence bound:
\[
\forall x,y \in [p,1-p], \qquad \KL\bigl(\Ber(x),\Ber(y)\bigr) \leq \frac{(x-y)^2}{2p(1-p)}\,.  \vspace{-.125cm}
\]
Their construction may actually be generalized to models $\cD$ with $C_{\cD} > 0$ such that for
all $\nu,\nu'$ in $\cD$, one has $\smash{\KL(\nu,\nu') \leq C_{\cD} \bigl( \Ed(\nu) - \Ed(\nu') \bigr)^2}$. This is a property
that clearly holds for some exponential families: on top
of the restricted Bernoulli model discussed above, for which
\[
C_{\cDberp} = 1/\bigl( 2p(1-p) \bigr)\,,
\]
we may cite the model $\cD_{\sigma^2}$ of Gaussian distributions with variance $\sigma^2$,
for which $C_{\cD_{\sigma^2}} = 1/(2\sigma^2)$. For models enjoying the existence
of such a constant $C_{\cD}$, (a straightforward modification of)
the analysis by \citet{ABM10} entails that for any $\unu$ in $\cD$,
\begin{equation}
\label{eq:LB-ABM10}
\liminf_{T \to +\infty} \frac{1}{T} \ln \P_{\unu}\bigl(I_T \ne a^\star(\unu)\bigr) \geq
- 5 \, C_{\cD} \min_{2 \leq k \leq K} \frac{\Delta_{(k)}^2}{k}\,.  \vspace{-.075cm}
\end{equation}
As by the very assumption on the model, $\Linfstr\bigl(\mu_{(k)}, \nu^\star\bigr) \leq C_\cD \, \Delta_{(k)}^2$,
the lower bound~\eqref{eq:ov_LB2-bis} implies the stated lower bound~\eqref{eq:LB-ABM10}, with an improved constant factor.

The lower bound~\eqref{eq:LB-ABM10} and the upper bound~\eqref{eq:UB_ABM10}
differ in particular by a factor proportional to $\bln K$.
\citet{CL16} discuss this gap in the case of the Bernoulli model $\cDber$
and improve the lower bound~\eqref{eq:LB-ABM10} by a factor of $\ln K$,
but not simultaneously for all bandit problems $\unu$ (as we aim for); they
obtain the improvement just for one bandit problem $\unu$. Their lower bound result
(formally stated and discussed in Appendix~\ref{app:CL16})
is therefore of a totally different nature.
More results on how and when given lower bounds with a given
complexity measure may, or may not, be improved
were stated by~\citet{KTH22}.

\paragraph{Discussion of the non-parametric bound for $K=2$ arms of~\citet{KCG16}.}
It turns out that the existing literature for the fixed-budget setting
offered so far a non-parametric bound, in the case of $K=2$
arms. Namely, in a general, possibly non-parametric model $\cD$,
\citet[Theorem~12]{KCG16} stated a lower bound for all $2$--armed bandit problems $\unu = (\nu_1,\nu_2)$:
\begin{equation} \label{eq:LB-KCG16-2arms}
\liminf_{T \to +\infty} \frac{1}{T} \ln \P_{\unu}\bigl(I_T \ne a^\star(\unu) \bigr)
\geq -
\hspace{-.4cm}
\inf_{\substack{\ula \text{ in } \cD \,: \\ \Ed(\lambda_{a^\star(\unu)}) < \Ed(\lambda_{w_\star(\unu)})}}
\hspace{-.4cm}
\max \Bigl\{ \KL\bigl(\lambda_{w_\star(\unu)}, \nu_{w_\star(\unu)}\bigr), \,
\KL\bigl(\lambda_{a^\star(\unu)}, \nu_{a^\star(\unu)} \bigr) \Bigr\}\,,
\end{equation}
where $w_\star(\unu)$ denotes the suboptimal arm in $\unu$ and
where the infimum is over all alternative bandit problems $(\lambda_1,\lambda_2)$ in $\cD$
with reverse order on the expectations compared to $\unu$.
We note (see the proof of Theorem~\ref{th:LB-conj-GK16}) that we may actually rewrite this
lower bound in a more readable way, in terms of $\Linfstr$ and $\Lsupstr$ quantities,
illustrating once again that these quantities are key in measuring the complexity of
best-arm identification under a fixed budget:
\begin{multline}
\label{eq:D2bras}
\inf_{\substack{\ula \text{ in } \cD \,: \\ \Ed(\lambda_{a^\star(\unu)}) < \Ed(\lambda_{w_\star(\unu)})}}
\hspace{-.4cm}
\max \Bigl\{ \KL\bigl(\lambda_{w_\star(\unu)}, \nu_{w_\star(\unu)}\bigr), \,
\KL\bigl(\lambda_{a^\star(\unu)}, \nu_{a^\star(\unu)} \bigr) \Bigr\} \\[-.5cm]
= \inf_{x \in [\mu_{w_\star(\unu)}, \mu^\star]}
\biggl\{ \max \Bigl\{ \Lsupstr\bigl(x, \nu_{w_\star(\unu)}\bigr), \, \Linfstr\bigl(x, \nu^\star\bigr) \Bigr\} \biggr\}\,.
\end{multline}
The proof technique of~\citet{KCG16} may be applied in a pairwise fashion to generalize the
lower bound~\eqref{eq:D2bras} for $2$ arms into a lower bound for $K \geq 2$ arms, stated in Theorem~\ref{th:LB-conj-GK16}:
for all $\unu$ in $\cD$ with a unique optimal arm,
\begin{equation}
\label{eq:D2bras-gener}
\liminf_{T \to +\infty} \frac{1}{T} \ln \P_{\unu} \bigl(I_T \neq a^\star(\unu)\bigr) \geq
- \min_{k \neq a^\star(\unu)}
\inf_{x \in [\mu_k, \mu^\star]} \Bigl\{ \max \bigl\{ \Lsupstr(x, \nu_k), \, \Linfstr(x, \nu^\star) \bigr\} \Bigr\} \,.
\end{equation}
We however do not claim that~\eqref{eq:D2bras-gener} is a deep and interesting bound, as it
only involves pairwise comparisons with the best arm. In particular, we lack
divisions by the ranks of the arms, as in~\eqref{eq:ov_LB2}.
This is why we had not stated the result~\eqref{eq:D2bras-gener} of Theorem~\ref{th:LB-conj-GK16} in
Section~\ref{subsec:overview} and mention it only here.

That being said, given that the infima in~\eqref{eq:ov_LB2} are over more restricted ranges
than in~\eqref{eq:D2bras-gener}, we can see no obvious ranking between the two bounds, which rather look incomparable.

\paragraph{Bounds for $K=2$ arms and exponential families, cf.\
comments after Theorem~12 of~\citet{KCG16}.}
We denote by $\cDexp$ the model corresponding to a
canonical one-parameter exponential family with expectations defined on an open interval $\cM$ (see Appendix~\ref{app:cDexp} for a reminder
on this matter). For such a model, we denote by $d$ the mean-parameterized Kullback-Leibler divergence.
By continuity of $d$, we have that for all $\nu$ in $\cDexp$ and for all $x \in \cM$,
\begin{align}
& \forall x \leq \Ed(\nu), \qquad \Linfstr(x,\nu) = \Linfeq(x,\nu) = d\bigl(x, \Ed(\nu)\bigr)\,, \label{eq:cont-d1} \\
\mbox{and} \qquad &
\forall x \geq \Ed(\nu), \qquad \Lsupstr(x,\nu) = \Lsupeq(x,\nu) = d\bigl(x, \Ed(\nu)\bigr) \label{eq:cont-d2} \,.
\end{align}
Note that all bounds stated in Section~\ref{subsec:overview} then admit simple reformulations in terms of $d$.
The Chernoff-information-type quantity $\chern$ introduced in~\eqref{eq:defchern}
may also be mean-parameterized as follows: for $\mu' < \mu$,
\begin{equation} \label{eq:Lchernoff_Dexp}
L(\mu',\mu) = \min_{x \in [\mu',\mu]} \bigl\{ d(x,\mu') + d(x,\mu) \bigr\}\,.
\end{equation}
We now explain why we called $L$ (and therefore $\chern$) a version of Chernoff information.
The original definition of the Chernoff information
$D(\mu', \mu)$ is the value $d(y, \mu)$ for $y \in [\mu', \mu]$ such that $d(y, \mu') = d(y, \mu)$.
As mentioned in the comments after Theorem~12 of~\citet{KCG16},
$D$ is the quantity at stake in~\eqref{eq:D2bras} for a canonical one-parameter exponential family: given that
$d(\,\cdot\,,\mu')$ and $d(\,\cdot\,,\mu)$ are respectively increasing and decreasing
on $[\mu',\mu]$,
\[
\min_{x \in [\mu',\mu]} \max\bigl\{ d(x,\mu'), \, d(x,\mu) \bigr\} = D(\mu', \mu)\,.
\]
Therefore, $D(\mu', \mu) \leq L(\mu',\mu) \leq 2\,D(\mu', \mu)$, which shows that $L$ is related to $D$,
as claimed.

\begin{example}
\label{ex:L-Ber}
We state the lower bound~\eqref{eq:ov_LB2-bis} and the upper bound~\eqref{eq:ov_boundSR}
for the model $\cDberp$ of Bernoulli distributions with parameters in $[p,1-p]$, where $p \in (0, 1/2)$.
We denote by
\[
\kl(x,y) = x \ln \frac{x}{y} + (1-x) \ln \frac{1-x}{1-y}\,, \qquad \mbox{where} \qquad x,y \in [p,1-p]
\]
the mean-parameterized Kullback-Leibler divergence of this model. We consider a generic bandit problem
$\unu = \bigl( \Ber(p_1), \ldots, \Ber(p_K) \bigr)$. We rank the parameters as in~\eqref{eq:ov_LB2},
i.e., introduce the notation $p^\star = p_{(1)} > p_{(2)} > \ldots > p_{(K)}$. Then, after noticing
(see Lemma~\ref{lm:monotony_chernoff_Dexp} in Appendix~\ref{app:cDexp}) that this ranking is the same as
the one considered in~\eqref{eq:ordering_by_L}, the upper bound~\eqref{eq:ov_boundSR} rewrites as
\[
\limsup_{T \to +\infty} \frac{1}{T} \ln \P \bigl( I_T \ne a^\star(\unu) \bigr)
\leq - \frac{1}{\bln K} \min_{2 \leq k \leq K} \frac{\displaystyle{\min_{x \in [p_{(k)}, p^\star]}}
\Bigl\{ \kl\bigl(x, p_{(k)}\bigr) + \kl(x, p^\star) \Bigr\}}{k} \,,
\]
while the lower bound~\eqref{eq:ov_LB2-bis} rewrites as
\[
\liminf_{T \to +\infty} \frac{1}{T} \ln \P_{\unu}\bigl(I_T \ne a^\star(\unu) \bigr)
\geq - \min_{2 \leq k \leq K} \frac{\kl\bigl( p_{(k)}, p^\star\bigr)}{k} \,.
\]
They should be compared to
the upper~\eqref{eq:UB_ABM10} and lower~\eqref{eq:LB-ABM10} bounds of \citet{ABM10}, respectively.
\end{example}

\subsection{Discussion of the (lack of) optimality of the new bounds exhibited}
\label{sec:discuss-opt}

The lower bound~\eqref{eq:ov_LB2} does not match the upper bound~\eqref{eq:ov_boundSR}
because of two aspects.
First, the infima in~\eqref{eq:ov_LB2} are only taken on restricted ranges
$[ \mu_{(k)}, \mu_{(k-1)} )$ and not on the entire intervals
$[ \mu_{(k)}, \mu^\star]$ as in~\eqref{eq:ov_boundSR}.
Second, the upper bound~\eqref{eq:ov_boundSR} involves a $1/\bln K$ factor,
while the lower bound~\eqref{eq:ov_LB2} does not.
A similar $1/\bln K$ factor was missing between
the upper~\eqref{eq:UB_ABM10} and lower~\eqref{eq:LB-ABM10} bounds of \citet{ABM10}
for Bernoulli models, together with a numerical factor of $5 \, C_{\cDberp}$.
The non-parametric bounds exhibited in this article mainly generalize and extend the known parametric
bounds but do not refine the latter in the sense that gaps between upper and lower bounds would be closed.

That being said, we would like to illustrate below on one specific example
to which extent the gap-based bounds can be looser.

\paragraph{Example of an extreme improvement: distributions with separated supports.}
For general non-parametric models, gaps are not enough at all to measure complexity as we may well have a finite gap between two distributions $\nu_1$ and $\nu_2$ with $\mu_1 > \mu_2$, but $\mathcal{L}(\nu_2, \nu_1) = +\infty$. This holds, for instance, as soon as $\nu_1$ and $\nu_2$ have
closed supports separated by a threshold $x_0$, i.e., the closed supports of $\nu_1$ and $\nu_2$ are included in $(-\infty,x_0)$ and
$(x_0,+\infty)$, respectively. Indeed, by mimicking the beginning of the proof of Lemma~\ref{lm:Linfstr-egal-Linfeq} of Appendix~\ref{sec:spec-cDall},
it may be seen that $\Linfeq(x, \nu_1) = +\infty$ for $x \leq x_0$ and $\Lsupeq(x, \nu_2) = +\infty$ if $x \geq x_0$, so that in all cases,
the sum $\Lsupeq(x, \nu_2) + \Linfeq(x, \nu_1)$ equals $+\infty$, and thus, $\mathcal{L}(\nu_2, \nu_1) = +\infty$.
In our bounds, e.g., the upper bound~\eqref{eq:ov_boundSR}, the pair of distributions $\nu_1,\nu_2$
will therefore not contribute---as intuition commands: these two distributions are easy to distinguish---,
while it does contribute in the earlier gap-based bounds.

\section{Upper bound: successive-rejects strategy, with an improved analysis}
\label{sec:SR}

We consider the successive-rejects strategy introduced by \cite{ABM10},
for $K$ arms and a budget $T$. The strategy works in phases, and
the lengths of the phases are set beforehand; they are denoted by $\ell_1,\ldots,\ell_{K-1} \geq 1$
and satisfy $\ell_1 + \ldots + \ell_{K-1} = T$. The strategy maintains a list of candidate arms,
starting with all arms, i.e., $S_0 = \{1,\ldots,K\}$. At the end of
each phase $r \in \{1,\ldots,K-1\}$, it drops an arm to get $S_r$,
while during phase $r$, it operates with the $K-r+1$ arms in $S_{r-1}$.

More precisely, during pahse $r \in \{1,\ldots, K-1\}$, the strategy draws $\lfloor \ell_r / (K-r+1) \rfloor$ times
each arm in $S_{r-1}$ (and does not use the few remaining time steps, if there are some).
At the end of each phase~$r$, the strategy computes the empirical averages $\ol{X}_a^r$
of the payoffs obtained by each arm $a \in S_{r-1}$ since the beginning; i.e., $\ol{X}_a^r$
is an average over
\[
\smash{N_r = \lfloor \ell_1 / K \rfloor + \ldots + \lfloor \ell_r/(K-r+1) \rfloor}
\]
i.i.d.\ realizations of $\nu_a$.
It then drops the arm $a_r$ with smallest empirical average (ties broken arbitrarily).
This description is summarized in the algorithm box.

\begin{figure}[t]
\begin{nbox}[title=Algorithm: successive-rejects strategy]
\textbf{Parameters:} $K$ arms, budget $T$, lengths $\ell_1,\ldots,\ell_{K-1} \geq 1$
with $\ell_1 + \ldots + \ell_{K-1} = T$ \medskip

\textbf{Initialization:} $S_0 = \{1,\ldots,K\}$ \medskip

\textbf{For each phase} $r \in \{1,\ldots,K-1\}$\textbf{:}
\begin{enumerate}
\item For each arm $a \in S_{r-1}$
\begin{enumerate}
\item Pull it $\lfloor \ell_r / (K-1+r) \rfloor$ times
\item Compute the empirical average $\ol{X}_a^r$ of the payoffs obtained in this phase
and in the previous phases
\end{enumerate}
\item Drop the arm $a_r$ with smallest average (ties broken arbitrarily):
\[
S_r = S_{r-1} \setminus \{ a_r \}\,,
\qquad \mbox{where} \qquad
a_r \in \argmin_{a \in S_{r-1}} \ol{X}_a^r
\]
\end{enumerate}

\textbf{Output:} Recommend arm $I_T$, where $S_{K-1} = \{ I_T \}$
\end{nbox}
\end{figure}

\subsection{General analysis}
\label{sec:SR-gal}

The key quantities for the general analysis will be
the logarithmic moment-generating function $\phi_{\nu}$ of a distribution $\nu \in \cD$,
and its Fenchel-Legendre transform $\phi^{\star}_{\nu}$:
\begin{equation}	\label{eq:defphistar}
\forall \lambda \in \R, \quad \phi_{\nu}(\lambda) = \ln \int_\R \! \e^{\lambda x} \,\mathrm{d}\nu(x)
\qquad \mbox{and} \qquad
\forall x \in \R, \quad \phi^\star_{\nu}(x) = \sup_{\lambda \in \R} \bigl\{ \lambda x - \phi_{\nu}(\lambda) \bigr\} \,. \vspace{-.15cm}
\end{equation}
Based on them, we can now define, for all $\nu, \nu' \in \cD$ with $\Ed(\nu') < \Ed(\nu)$,
\[
\smash{\Phi(\nu',\nu) \eqdef \inf_{x \in [\Ed(\nu'), \Ed(\nu)]} \bigl\{ \phi^\star_{\nu'}(x) + \phi^\star_{\nu}(x) \bigr\}} \,.
\]
The following simple lemma shows that $\Phi$
plays a significant role for bounding the probability that two sample averages are in reverse
order compared to the expectations of the underlying distributions.
It supersedes the use of Hoeffding's inequality in~\cite{ABM10}.

\begin{restatable}{lem}{grille}
\label{lem:devphistar}
Fix $\nu$ and $\nu'$ in $\cD$, with respective expectations $\mu = \Ed(\nu) > \mu' = \Ed(\nu')$.
For all $N \geq 1$, let $\ol{X}_N$ and $\ol{Y}_N$ be the averages of $N$--samples with respective distributions $\nu$ and $\nu'$.
Then,
\[
\limsup_{N \to +\infty} \frac{1}{N} \ln \P \bigl( \overline{X}_N \leq \overline{Y}_N \bigr)
\leq - \inf_{x \in [\mu', \mu]} \bigl\{ \phi^\star_{\nu'}(x) + \phi^\star_{\nu}(x) \bigr\} \, \eqdef - \Phi(\nu',\nu)\,.
\]
\end{restatable}

\begin{proof}{\sk}
The fact $\overline{X}_N \leq \overline{Y}_N$ entails the existence of $x$ such that $\overline{X}_N \leq x \leq \overline{Y}_N$.
By independence, together with two applications of the Cram{\'e}r-Chernoff bound (recalled
in Appendix~\ref{sec:app-SR-CR}),
\[
\P \bigl( \overline{X}_N \leq x \leq \overline{Y}_N \bigr)
= \P \bigl( \overline{X}_N \leq x\bigr) \, \P \bigl( x \leq \overline{Y}_N \bigr)
\leq \exp \bigl( - N \, \phi^\star_{\nu}(x) \bigr) \,\, \exp \bigl( - N \, \phi^\star_{\nu'}(x) \bigr)\,.
\]
The technical issue is then to deal with some union over~$x$ of the events $\bigl\{\overline{X}_N \leq x \leq \overline{Y}_N\bigr\}$.
We do so with a sequence of finite grids, with vanishing steps, and use lower-semi-continuity arguments to obtain an infimum over an interval based on a sequence of finite minima. A complete proof is to be found in Appendix~\ref{sec:lmdevphistar}.
\end{proof}

The main performance upper bound is stated below in terms of $\Phi$, that is,
in terms of Fenchel-Legendre transforms of logarithmic moment-generating functions.
Section~\ref{sec:Phi=L} will later explain why and when the latter may be replaced
by $\Linfeq$ and $\Lsupeq$ quantities, leading to a rewriting $\Phi = \chern$
and to the bound claimed in~\eqref{eq:ov_boundSR}.

\begin{theorem} \label{th:SR}
Fix $K\geq 2$ and a model $\cD$. Consider a sequence of successive-rejects strategies, indexed by $T$,
such that $N_r/T \to \gamma_r > 0$ as $T \to +\infty$ for all $r \in \{1, \dots, K-1\}$.
Let $\unu$ be a bandit problem in $\cD$ with a unique optimal arm and, for each $r \in \{1, \dots, K-1\}$, let
$\cA_r$ be a subset of arms of cardinality $r$ that does not contain $a^\star(\unu)$. Then
\[
\limsup_{T \to +\infty} \frac{1}{T} \ln \P \bigl( I_T \ne a^\star(\unu) \bigr)
\leq - \min_{1 \leq r \leq K-1} \Bigl\{ \gamma_r \, \min_{k \in \cA_r}  \, \Phi \bigl( \nu_k, \nu^\star \bigr) \Bigr\} \,.
\]
\end{theorem}

\begin{proof}{\sk}
A complete proof may be found in Appendix~\ref{sec:thSR};
it mimics the analysis by~\citet{ABM10}, the main modification being the substitution of
Hoeffding's inequality by the bound of Lemma~\ref{lem:devphistar}.
We have $I_T \ne a^\star(\unu)$ if and only if $a^\star(\unu)$ is rejected in some phase,
i.e.,
\[
\bigl\{ I_T \ne a^\star(\unu) \bigr\}
= \bigcup_{r=1}^{K-1} \bigl\{ a_r = a^\star(\unu) \bigr\}
\subseteq \bigcup_{r=1}^{K-1} \Bigl\{ a^\star(\unu) \in S_{r-1} \ \ \mbox{and} \ \
\forall k \in S_{r-1}, \ \ \ol{X}^r_{a^\star(\unu)} \leq \ol{X}^r_{k} \Bigr\} \, .
\]
By optional skipping (see \citealp[Chapter III, Theorem 5.2, p. 145]{doob1953})
and by the fact that by the pigeonhole principle,
the (random) set $S_{r-1}$ necessarily contains one element of the deterministic set $\cA_r$,
\[
\P\Bigl( a^\star(\unu) \in S_{r-1} \ \ \mbox{and} \ \
\forall k \in S_{r-1}, \ \ \ol{X}^r_{a^\star(\unu)} \leq \ol{X}^r_{k} \Bigr)
\leq \sum_{k \in \cA_r} \P\Bigl( \ol{Y}^r_{a^\star(\unu)} \leq \ol{Y}^r_{k} \Bigr)\,,
\]
where, for all $a$, the $\ol{Y}_a^r$ are the averages of independent $N_r$--samples distributed according to $\nu_a$.
The proof is concluded by
Lemma~\ref{lem:devphistar} and the fact that a sum of exponentially fast decaying quantities
is driven by its largest term.
\end{proof}

We conclude this subsection by stating the bound of
Theorem~\ref{th:SR} for the phase lengths suggested by~\citet{ABM10}, namely,
$\ell_1 = T/\bln K$ and for $r \in \{2, \ldots, K-1\}$,
\begin{equation}
\label{eq:reglengths-ABM10}
\ell_r = \frac{T}{(K-r+2) \bln K}\,,
\qquad \mbox{where} \qquad \bln K = \frac{1}{2} + \sum_{k=2}^K \frac{1}{k}\,.
\end{equation}
We also consider lower bounds $f\bigl( \nu_k, \nu^\star \bigr)$
on the $\Phi\bigl( \nu_k, \nu^\star \bigr)$. We may of course use $f = \Phi$
but sometimes, it is handy to rely on more readable lower bounds.
For instance, in the case of the $\cDall$ model, Hoeffding's inequality entails that
\newcommand{\boundphiHbd}{\phi^\star_{\nu}(x) \geq 2\bigl(x-\Ed(\nu)\bigr)^2\,, \qquad \mbox{so that} \qquad
\Phi\bigl( \nu_k, \nu^\star \bigr) \geq \Delta_k^2 \eqdef f\bigl( \nu_k, \nu^\star \bigr)}
\begin{equation}
\label{eq:phistarHbd}
\boundphiHbd\,;
\end{equation}
see more details in Appendix~\ref{sec:cor-SR}.
Such bounds hold more generally in models consisting of sub-Gaussian distributions.

We now order the arms into $\sigma_1,\ldots,\sigma_K$ based on $f$, namely, we let $\sigma_1 = a^\star(\unu)$ and
\begin{equation}
\label{eq:orderingU}
0 = f\bigl( \nu_{\sigma_1}, \nu^\star \bigr)
< f\bigl( \nu_{\sigma_2}, \nu^\star \bigr) \leq \ldots
\leq f\bigl( \nu_{\sigma_{K-1}}, \nu^\star \bigr) \leq f\bigl( \nu_{\sigma_K}, \nu^\star \bigr)\,,
\end{equation}
and we take $\cA_r = \{\sigma_{K-r+1},\ldots,\sigma_K\}$.
We obtain immediately the following corollary, for which a detailed proof may be found, for the sake of completeness,
in Appendix~\ref{sec:cor-SR}.

\begin{corollary} \label{cor:SR} Fix $K\geq 2$, a model $\cD$,
and consider a lower bound $f$ on $\Phi$. The sequence of successive-rejects strategies based
on the phase lengths~\eqref{eq:reglengths-ABM10} ensures, that
for all bandit problems $\unu$ in $\cD$ with a unique optimal arm,
\[
\limsup_{T \to +\infty} \frac{1}{T} \ln \P \bigl( I_T \ne a^\star(\unu) \bigr)
	\leq - \frac{1}{\bln K} \min_{2 \leq k \leq K} \frac{f\bigl( \nu_{\sigma_{k}}, \nu^\star \bigr)}{k}\,,
\]
where arms were reordered as in~\eqref{eq:orderingU}.
\end{corollary}

\subsection{On links between $\Phi$ and the quantities $\Linfstr$, $\Linfeq$, $\Lsupstr$ and $\Lsupeq$}
\label{sec:Phi=L}

The Fenchel-Legendre transform $\phi^{\star}_{\nu}$ of the logarithmic moment-generating function of $\nu$ admits a classical (see, e.g., \citealp[Exercice 4.13]{BLM16}) dual formulation in terms of infima of Kullback-Leibler divergences.
The following lemma, proved in Appendix~\ref{sec:appcDall}, reveals that
these infima correspond to $\Linfeq$ and $\Lsupeq$
for the model $\cDall$ of distributions supported on $[0, 1]$.

\begin{restatable}{lem}{duality}
\label{lem:duality}
Consider the model $\cD = \cDall$.
For all $\nu \in \cDall$,
\[
\forall x \leq \Ed(\nu), \quad \phi^\star_\nu(x) = \Linfeq(x,\nu)
\qquad \mbox{and} \qquad
\forall x \geq \Ed(\nu), \quad \phi^\star_\nu(x) = \Lsupeq(x,\nu) \,.
\]
\end{restatable}

\noindent
Based on this lemma, we have the following rewriting, which is useful to reinterpret the quantities
appearing in Theorem~\ref{th:SR} and Corollary~\ref{cor:SR}:
$\Phi(\nu',\nu) = \chern(\nu', \nu)$ for the model $\cDall$, i.e.,
\begin{equation} \label{eq:Phi=L}
\inf_{x \in [\Ed(\nu'), \Ed(\nu)]} \bigl\{ \phi^\star_{\nu'}(x) + \phi^\star_{\nu}(x) \bigr\} = \inf_{x \in [\Ed(\nu'), \Ed(\nu)]} \bigl\{ \Lsupeq(x, \nu') + \Linfeq(x, \nu) \bigr\} \,.
\end{equation}
For canonical one-parameter exponential models $\cDexp$,
a slightly weaker version of Lemma~\ref{lem:duality}, only holding for $x$ corresponding to expectations in~$\cDexp$
and provided in Appendix~\ref{app:cDexp}, similarly shows~\eqref{eq:Phi=L}, i.e., $\Phi = \chern$.
Conditions on general models for $\Phi = \chern$ to hold are discussed in Appendix~\ref{app:Phi=L_condtion}.

\section{Lower bounds}
\label{sec:LB}

In most of this section, we restrict our attention to generic $K$--armed bandit problems $\unu$,
that are such that $\mu_j \ne \mu_k$ for $j \ne k$.
In particular, the best arm $a^\star(\unu)$ is unique.
(This is probably a new terminology\footnote{The terminology comes from measure theory:
if expectations were drawn at random according to some diffuse distribution, e.g., a uniform distribution
over an interval, or a Gaussian distribution, then, almost surely, no two expectations would be
equal.} for referring to bandit problems with no two same expectations
for the distributions over the arms.)

\paragraph{Definition of a strategy, and of a (doubly-indexed) sequence of strategies.}
A strategy $(\psi,\varphi)$ depends on the budget $T$ and the number $K$ of arms;
it consists of a sampling scheme $\psi = (\psi_t)_{1 \leq t \leq T}$
and a recommendation function $\varphi$.
At each round $t \in \{1, \dots, T\}$, the strategy picks an arm
$A_t$, possibly at random using an auxiliary randomization $U_{t-1}$.
Given this choice $A_t$, the strategy
observes a payoff $Y_t$ drawn at random according to $\nu_{A_t}$, independently from the past.
For $t \geq 2$,
the choice $A_t$ is therefore a measurable function $A_t = \psi_t(H_t)$
of the history $H_t = (U_0, \, Y_1, \, \ldots, \, Y_{t-1}, \, U_{t-1})$,
while $A_1 = \psi_1(H_0)$, where $H_0 = U_0$.
At round $T$, the strategy recommends the arm $I_T = \varphi(H_T)$.

\paragraph{Outline of this section.}
As always in lower-bound results, there is a trade-off between how restrictive are the assumptions
on the (doubly-indexed) sequences of strategies, and sometimes on the models, and how large the lower bounds are:
the more restrictive the assumptions, the larger the lower bounds.
We are interested in assumptions on strategies that are natural in the sense that they should
be satisfied by successive-rejects-type strategies.
For instance, Theorem~\ref{th:LB-conj-GK16} comes with the least assumptions
but provides a bound where there are no divisions by the ranks $k$ of the arms,
which Theorems~\ref{th:LB-ABM10} and~\ref{th:LBmonotone} do.
We may see Theorem~\ref{th:LB-ABM10} as a warm-up result: its main aim is to generalize the
lower bound by~\citet{ABM10} to non-parametric models with a (non-constructive) proof that is only a few-line
long. Our preferred result is Theorem~\ref{th:LBmonotone},
which provides the largest lower bound while putting the heaviest (though natural) constraints on the sequences of strategies.

\subsection{Common restriction: consistence}

For our lower bounds,
we will consider sequences of strategies, either only
indexed by $T \geq 1$ given a value of $K \geq 2$,
or doubly indexed by $T$ and $K$.
These sequences will also be assumed to be ``reasonable''
in the sense below.

\paragraph{Consistent (or exponentially consistent) sequences of strategies.}
The probability $\smash{\P\bigl(I_T \ne a^\star(\unu)\bigr)}$
of misidentifying the unique optimal arm may vanish
asymptotically (and even vanish exponentially fast)
for all bandit problems---in not too large a model $\cD$,
as illustrated in Section~\ref{sec:SR}.
We will therefore only be interested in such sequences of strategies, called
(exponentially) consistent.
In the sequel and for extra clarity, we index
the probabilities by the ambient bandit problem $\unu$ considered.

\begin{definition}
\label{def:exp-consistent}
Fix $K \geq 2$. A sequence of strategies indexed by $T \geq 1$
is consistent, respectively, exponentially consistent, on a model $\cD$ if
for all generic problems $\unu$ in $\cD$,
\[
\P_{\unu}\bigl(I_T \ne a^\star(\unu)\bigr) \lriT 0\,, \qquad \mbox{respectively,} \qquad
\smash{\limsup_{T \to +\infty} \frac{1}{T}} \ln
\P_{\unu}\bigl(I_T \ne a^\star(\unu)\bigr) < 0\,.
\]
By extension, a doubly-indexed sequence of strategies is
(exponentially) consistent if for all $K \geq 2$, the associated
sequences of strategies are so.
\end{definition}

\paragraph{The fundamental inequality.}
The fundamental inequality by~\citet{GMS19}, together with the very
definition of consistency, yields in a straightforward manner
our building block for lower bounds. Details of the derivation are
provided in Appendix~\ref{app:lem:BI}, for the sake of completeness.

\begin{restatable}{lem}{lmfunda}
\label{lem:BI}
Fix $K \geq 2$ and a model $\cD$. Consider a consistent sequence of strategies on $\cD$, and two generic bandit problems $\unu$ and $\ula$ in $\cD$ such that $a^\star(\ula) \neq a^\star(\unu)$. Then
\begin{multline*}
\liminf_{T \to +\infty} \frac{1}{T} \ln \P_{\unu}\bigl(I_T \neq a^\star(\unu)\bigr)
\geq - \limsup_{T \to +\infty} \sum_{a = 1}^K \frac{\E_{\ula}[N_a(T)]}{T} \, \KL(\lambda_a, \nu_a) \,, \\
\mbox{where} \qquad
\smash{N_a(T) = \sum_{t=1}^T \ind{A_t = a}}
\end{multline*}
denotes the number of times arm~$a$ was pulled in the $T$ exploration rounds
of a given strategy with budget~$T \geq 1$.
\end{restatable}

\subsection{A lower bound revisiting and extending the one by \citet{ABM10}}
\label{sec:LB-ABM10}

The focus of this subsection is to establish the lower bound~\eqref{eq:ov_LB2-bis},
from which we derived the gap-based lower bound~\eqref{eq:UB_ABM10} by \citet{ABM10}.
The lower bound~\eqref{eq:ov_LB2-bis} is smaller than the lower bound to be exhibited
in the next subsection, but it comes with less restrictive assumptions
on the behaviors of the sequences of strategies considered.

Firstly, we only consider sequences of strategies---actually, sequences of sampling
schemes---that do not pull too often the worst arm, and which we will refer to
as being balanced against the worst arm. Successive-rejects-type strategies sample the worst arm less
than other arms in expectations, and hence, are indeed balanced against the worst arm.
To define this constraint formally, we denote by $w_\star(\unu)$ the index of the unique worst arm
of a generic bandit problem $\unu$.

\begin{definition}
\label{eq:balanced}
A doubly-indexed sequence of strategies is balanced against the worst arm on a model $\cD$
if for all $K \geq 2$, for all generic $K$--armed bandit problems $\unu$ in $\cD$, \vspace{-.2cm}

\[
\smash{\limsup_{T \to +\infty} \frac{1}{T}} \, \E_{\unu}\bigl[ N_{w_\star(\unu)}(T) \bigr]
\leq \frac{1}{K}\,.
\]
\end{definition}

A second constraint is related to bandit subproblems.
We say that $\unu'$ is a subproblem of a $K$--armed bandit problem $\unu$ if $\unu' = (\nu_a)_{a \in \cA}$
for a subset $\cA \subseteq \{1,\ldots,K\}$ of cardinality greater than or equal to $2$; we
denote by $\unu' \subseteq \unu$ this fact. We say in addition that
$\unu'$ and $\unu$ feature the same optimal arm if
$\nu'_{a^\star(\unu')} = \nu_{a^\star(\unu)}$.
It should be easier to identify the best arm in $\unu'$ than in $\unu$,
in the sense below, and this defines the fact that
a strategy cleverly exploits pruning of suboptimal arms.
Again, successive-rejects-type strategies naturally satisfy this constraint.

\begin{definition}
A doubly-indexed sequence of strategies cleverly exploits pruning of suboptimal arms on a model $\cD$ if
for all generic bandit problems $\unu$ in $\cD$ with $K \geq 2$ arms,
for all subproblems $\unu' \subseteq \unu$ featuring the same optimal arm, \vspace{-.2cm}

\[
\smash{\liminf_{T \to +\infty} \frac{1}{T}} \ln \P_{\unu}\bigl(I_T \ne a^\star(\unu)\bigr)
\geq
\smash{\liminf_{T \to +\infty} \frac{1}{T}} \ln \P_{\unu'}\bigl(I_T \ne a^\star(\unu')\bigr)\,.
\]
\end{definition}

\noindent
We use again the order statistics
$\mu_{w_\star(\unu)} = \mu_{(K)} < \mu_{(K-1)} < \ldots < \mu_{(1)} = \mu_{a^\star(\unu)}$.

\begin{restatable}{theo}{lbabm}
\label{th:LB-ABM10} Fix a model $\cD$. Consider a doubly-indexed sequence of strategies that is
consistent, balanced against the worst arm on $\cD$, and
that cleverly exploits the pruning of suboptimal arms on~$\cD$.
For all generic bandit problems $\unu$ in $\cD$ with $K \geq 2$ arms, \vspace{-.3cm}

\[
\liminf_{T \to +\infty} \frac{1}{T} \ln \P_{\unu}\bigl(I_T \ne a^\star(\unu)\bigr)
\geq - \min_{2 \leq k \leq K} \frac{\Linfstr\bigl(\mu_{(k)},\nu^\star\bigr)}{k}\,.
\]
\end{restatable}

\begin{proof}{\sk}
The bound is proved for $k=K$ by considering alternative bandit problems
$\ula$ differing from $\unu$ only at arm~$a^\star(\unu)$, where
$\nu^\star$ is replaced by distributions $\zeta \in \cD$ with $\Ed(\zeta) < \mu_{(K)}$.
For $\ula$, the arm~$a^\star(\unu)$ is the worst arm, and is therefore pulled less than a fraction $1/K$
of the time, asymptotically and on average, as the strategy is balanced against the worst arm.
An application of Lemma~\ref{lem:BI} concludes the case $k=K$.
The extension to $k \leq K-1$ is obtained by clever exploitation of the pruning of suboptimal arms.
A complete proof may be found in Appendix~\ref{app:th:LB-ABM10}.
\end{proof}

\subsection{A larger lower bound, for a more restrictive class of strategies} \label{sec:LB2}

In this section, we derive a slightly stronger version of the lower bound~\eqref{eq:ov_LB2}.
This lower bound is larger than the bound exhibited in the previous subsection but relies on stronger assumptions on the strategies
considered. Namely, we introduce an assumption of monotonicity, which extends Definition~\ref{eq:balanced}
to provide frequency constraints on each arm $a \in \{1,\ldots,K\}$.

\begin{definition}
\label{eq:monot}
Fix $K \geq 2$. A sequence of strategies is monotonous on a model $\cD$
if for all generic problems $\unu$ in $\cD$, for all arms $a \in \{1, \dots, K\}$,
\[
\limsup_{T \to +\infty} \frac{\E_{\unu} \bigl[ N_{(a)}(T) \bigr]}{T} \leq \frac{1}{a} \,,
\]
where arms are ordered such that $\mu_{(1)} > \mu_{(2)} > \dots > \mu_{(K)}$.
\end{definition}

This condition is satisfied as soon as a given arm is not pulled more often, asymptotically and on average, than
better-performing arms (note that Definition~\ref{eq:monot} is slightly weaker than this).
Successive-rejects-type strategies naturally satisfy this requirement.

We also rely on the following assumption on the model $\cD$,
which essentially indicates that there is ``no gap'' in $\cD$.
Once again, the model $\cDall$ and
canonical one-parameter exponential models $\cDexp$ all satisfy this mild requirement (see Appendix~\ref{sec:proof-normality}
for the immediate details).

\begin{restatable}{defi}{definormality}
\label{defi:normality}
A model $\cD$ is normal if for all $\nu \in \cD$, for all $x \geq \Ed(\nu)$,
\begin{align*}
\forall \varepsilon > 0, \qquad
\Lsupstr(x, \nu) & \eqdef
\inf \bigl\{ \KL(\zeta,\nu) : \ \zeta \in \cD \ \ \mbox{\rm s.t.} \ \ \Ed(\zeta) > x \bigr\} \\
& = \inf \bigl\{ \KL(\zeta,\nu) : \ \zeta \in \cD \ \ \mbox{\rm s.t.} \ \ x + \varepsilon > \Ed(\zeta) > x  \bigr\}\,.
\end{align*}
\end{restatable}

\begin{restatable}{theo}{thlbmonotone}
\label{th:LBmonotone}
Fix $K \geq 2$ and a normal model $\cD$.
Consider a sequence of strategies which is consistent and monotonous on $\cD$. For all generic bandit problems $\unu$ in $\cD$,
\[
\liminf_{T \to +\infty} \frac{1}{T} \ln \P_{\unu}\bigl(I_T \ne a^\star(\unu) \bigr)
\geq - \min_{2 \leq k \leq K} \,\, \min_{2 \leq j \leq k} \,\, \inf_{x \in [\mu_{(j)}, \mu_{(j-1)})}
\smash{\biggl\{ \frac{\Lsupstr\bigl(x, \nu_{(k)}\bigr)}{j-1} + \frac{\Linfstr\bigl(x, \nu^\star\bigr)}{j} \biggr\}} \,.
\]
\end{restatable}

\begin{proof}{\sk}
A complete proof may be found in Appendix~\ref{app:th:LBmonotone}.
For triplets $(k,j,x)$ satisfying the stated requirements,
we consider an alternative problem $\ula$ 
differing from the original bandit problem $\unu$ at the best arm $(1)$ and
at the $k$--th best arm $(k)$, for which we pick distributions such that
$\Ed\bigl(\lambda_{(1)}\bigr) < x < \Ed\bigl(\lambda_{(k)}\bigr) < \mu_{(j-1)}$.
Then arm $(1)$ is at best the $j$--th best arm of $\ula$, while arm $(k)$ is exactly the $j-1$--th
best arm of $\ula$. By monotonicity and Lemma~\ref{lem:BI},
we obtain
\begin{equation}
\label{eq:psk-monot}
\liminf_{T \to +\infty} \frac{1}{T} \ln \P_{\unu}\bigl(I_T \ne a^\star(\unu) \bigr)
\geq - \biggl( \frac{\KL\bigl(\lambda_{(k)}, \nu_{(k)}\bigr)}{j-1} + \frac{\KL\bigl(\lambda_{(1)}, \nu^\star\bigr)}{j} \biggr)\,.
\end{equation}
We get $- \Lsupstr\bigl(x, \nu_{(k)}\bigr)/(j-1) - \Linfstr\bigl(x, \nu^\star\bigr)/j$ as a lower bound by
taking (separate) suprema of the lower bound~\eqref{eq:psk-monot} over $\Ed\bigl(\lambda_{(1)}\bigr) < x$
and $x < \Ed\bigl(\lambda_{(k)}\bigr) < \mu_{(j-1)}$, where the $< \mu_{(j-1)}$ constraint disappears
thanks to normality of the model.
\end{proof}

\subsection{A general lower bound, valid for any strategy}
\label{sec:LB-conj-GK16}

The previous subsections illustrated what may be achieved under restrictions---though natural res\-tric\-tions---on the classes
of strategies considered. For the sake of completeness, we also provide a lower bound relying on no
other restriction than consistency; it extends the lower bound~\eqref{eq:LB-KCG16-2arms}
exhibited by~\cite{KCG16} for $K=2$ arms, and is formulated in terms of $\Linfstr$ and $\Lsupstr$.
A proof of the following theorem may be found in Appendix~\ref{app:thLBconjGK}.

\begin{restatable}{theo}{thLBconjGK}
\label{th:LB-conj-GK16}
Fix $K \geq 2$ and a model $\cD$. Consider a consistent sequence of strategies on $\cD$.
For all generic bandit problems $\unu$ in $\cD$, \vspace{-.2cm}

\[
\liminf_{T \to +\infty} \frac{1}{T} \ln \P_{\unu} \bigl(I_T \neq a^\star(\unu)\bigr) \geq - \min_{k \neq a^\star(\unu)} \,\,
\inf_{x \in [\mu_k, \mu^\star]} \,\, \max \bigl\{ \Lsupstr(x, \nu_k),  \Linfstr(x, \nu^\star) \bigr\} \,.
\]
\end{restatable}

\acks{Aurélien Garivier and Antoine Barrier acknowledge the support of the Project IDEXLYON of the University of Lyon, in the framework of the Programme
Investissements d'Avenir (ANR-16-IDEX-0005), and Chaire SeqALO (ANR-20-CHIA-0020-01).
We thank Hédi Hadiji for pointers relative to
the equality between $\phi^\star$ and $d$ in the case of exponential
models $\cDexp$.}

\bibliography{barrier23Bibliography}

\appendix
\section*{Content of the appendices}

\noindent
The appendices of this article contain the following elements.
\begin{itemize}
\item Appendix~\ref{app:Linf} states and proves some basic properties on quantities
$\Linfstr$, $\Linfeq$, $\Lsupstr$, and $\Lsupeq$ that were introduced in Section~\ref{sec:Linf}.
\item Appendix~\ref{sec:app-SR} provides the proofs for the first part of the analysis of the successive-rejects strategy,
namely, the general analysis in terms of $\Phi$, to be found in Section~\ref{sec:SR-gal}.
\item Appendix~\ref{app:phi=L} provides the proofs for the second part of the analysis of the successive-rejects strategy,
namely, the rewriting of $\Phi$ as $\chern$ that was the key contribution of Section~\ref{sec:Phi=L}.
\item Appendix~\ref{app:LB} is related to the lower bounds of Section~\ref{sec:LB}, and provides detailed proofs thereof.
\item Appendix~\ref{app:literature} contains additional elements on the literature review of Sections~\ref{sec:intro} and~\ref{sec:overview};
it states and discusses some important existing lower bounds.
\end{itemize}

\section{Properties of the $\Linfstr$, $\Linfeq$, $\Lsupstr$, and $\Lsupeq$ quantities}
\label{app:Linf}

We separate the list of properties in two categories:
general properties, that hold for all models $\cD$, in Appendix~\ref{sec:gen-prop-Linf};
specific properties for the model $\cD = \cDall$, in Appendix~\ref{sec:spec-cDall}.
It also worth noting that the
$\Linfstr$, $\Linfeq$, $\Lsupstr$, and $\Lsupeq$ quantities admit a simple rewriting in the case
of canonical one-parameter exponential models $\cDexp$, as the mean-parameterized Kullback-Leibler divergence~$d$, see Appendix~\ref{app:cDexp}.
Properties in this case thus follow from classical properties of~$d$.

\subsection{General properties}
\label{sec:gen-prop-Linf}

We state some properties for $\Linfstr$, that all also hold for $\Linfeq$; the corresponding properties for $\Lsupstr$
and $\Lsupeq$ are deduced by symmetry. \medskip

The function $\Linfstr(\,\cdot\,,\nu)$ is non-increasing and satisfies
$\Linfstr(x,\nu) = 0$ for all $x > \Ed(\nu)$,
as can be seen by taking $\zeta = \nu$.
Also, whenever $\cD$ is convex, the function $\Linfstr$ is jointly convex over $\R \times \cD$,
as indicated in the lemma below. In particular,
$x \mapsto \Linfstr(x,\nu)$ is continuous on the interior of its domain (the set
where it takes finite values).

\begin{lemma}
\label{lm:Linfcvx}
When $\cD$ is a convex model,
all four functions $\Linfstr$, $\Linfeq$, $\Lsupstr$, and $\Lsupeq$ are jointly convex
over $\R \times \cD$.
\end{lemma}

\begin{proof}
We provide the proof for $\Linfstr$, and it may be adapted in a straightforward manner for the other functions.

We set two distributions $\nu$ and $\nu'$ of $\cD$, two expectation levels $\mu$ and $\mu'$ in $\R$, and a weight $\lambda \in (0,1)$.
We want to prove that
\begin{equation}
\label{eq:desiredcvxineq}
\Linfstr\bigl(\lambda \mu + (1-\lambda) \mu', \lambda \nu + (1-\lambda) \nu'\bigr)
\leq \lambda \Linfstr(\mu, \nu) + (1-\lambda) \Linfstr(\mu', \nu')\,.
\end{equation}
The desired inequality holds whenever $\Linfstr(\mu, \nu) = +\infty$ or $\Linfstr(\mu', \nu') = +\infty$.
Otherwise, assuming that both $\Linfstr(\mu, \nu)$ and $\Linfstr(\mu', \nu')$ are finite,
we set $\delta > 0$ (which we will ultimately let converge to~$0$) and pick
$\zeta$ and $\zeta'$ in $\cD$ such that $\Ed(\zeta) < \mu$ and $\Ed(\zeta) < \mu'$, as well as
\[
\KL(\zeta, \nu) \leq \Linfstr(\mu, \nu) + \delta \quad \text{and} \quad \KL(\zeta', \nu') \leq \Linfstr(\mu', \nu') + \delta\,.
\]
Then, by joint convexity of the Kullback-Leibler divergence:
\begin{align*}
\lambda \Linfstr(\mu, \nu) + (1-\lambda) \Linfstr(\mu', \nu') + \delta & \geq \lambda \KL(\zeta, \nu) + (1-\lambda) \KL(\zeta', \nu') \\
& \geq \KL\bigl(\lambda \zeta + (1-\lambda) \zeta', \lambda \nu + (1-\lambda) \nu' \bigr) \\
& \geq \Linfstr\bigl(\lambda \mu + (1-\lambda) \mu', \lambda \nu + (1-\lambda) \nu' \bigr) \,,
\end{align*}
where for the last inequality, we used the definition of $\Linfstr$ as an infimum and
the fact that by convexity, the distribution $\lambda \zeta + (1-\lambda) \zeta'$ belongs to $\cD$, with expectation larger than $\lambda \mu + (1-\lambda) \mu'$.
The desired convexity inequality~\eqref{eq:desiredcvxineq} follows by letting $\delta \to 0$.
\end{proof}

\subsection{Specific properties for $\cD = \cDall$}
\label{sec:spec-cDall}

We now consider only the model $\cDall$ of
all distributions over $[0,1]$.

Since we are considering distributions over $[0,1]$,
the data-processing inequality for Kullback-Leibler divergences ensures
(see, e.g., \citealp[Lemma~1]{GMS19}) that
for all $\zeta \in \cDall$,
\[
\KL(\zeta,\nu) \geq \KL \Bigl( \Ber \bigl( \Ed(\zeta) \bigr), \, \Ber \bigl( \Ed(\nu) \bigr) \Bigr)
\geq 2 \bigl( \Ed(\zeta) - \Ed(\nu) \bigr)^2\,,
\]
where $\Ber(p)$ denotes the Bernoulli distribution with parameter~$p$
and where we applied Pinsker's inequality for Bernoulli distributions.
Therefore, taking the infimum over distributions $\zeta \in \cDall$ with $\Ed(\zeta) < x$,
\begin{equation}
\label{eq:Pinsker-Linf}
\forall x \leq \Ed(\nu), \quad \Linfstr(x,\nu) \geq 2 \bigl( \Ed(\nu) - x \bigr)^2\,.
\end{equation}

We denote by $m(\nu) = \min\bigl(\Supp(\nu)\bigr) \geq 0$
the minimum of the closed support $\Supp(\nu)$ of $\nu$; that is,
$m(\nu)$ is the largest value such that $\Supp(\nu) \subseteq \bigl[m(\nu),\,1\bigr]$.
We will refer to $m(\nu)$ as the lower end of the support of $\nu$.
Though we will not need it immediately, we also define
the upper end of the support of $\nu$ as
$M(\nu) = \max\bigl(\Supp(\nu)\bigr) \leq 1$; by symmetry, it will be considered when studying $\Lsupstr$ and $\Lsupeq$ instead of $\Linfstr$ and $\Linfeq$.

The lemma below states that the functions $\Linfstr(\,\cdot\,,\nu)$ and $\Linfeq(\,\cdot\,,\nu)$
coincide, except maybe at~$m(\nu)$.
One may wonder what happens at $x = m(\nu)$.
We denote by $\nu\bigl\{m(\nu)\bigr\}$ the probability mass assigned by $\nu$
to the point $m(\nu)$.
It follows from the second part the lemma below
that $\Linfstr\bigl(m(\nu),\nu\bigr) = \Linfeq\bigl(m(\nu),\nu\bigr)$
if and only if $\bigl\{m(\nu)\bigr\}$ is not an atom of $\nu$.

\begin{lemma}
\label{lm:Linfstr-egal-Linfeq}
We consider the model $\cD = \cDall$.
The function $\Linfstr(\,\cdot\,,\nu)$ is continuous on the interval $\bigl( m(\nu), \, +\infty\bigr)$.
We also have, on the one hand,
\begin{equation}
\label{eq:Linfstr-egal-Linfeq1}
\forall \mu \ne m(\nu), \qquad \Linfstr(\mu,\nu) = \Linfeq(\mu,\nu)\,,
\end{equation}
and on the other hand, at $\mu = m(\nu)$,
\begin{equation}
\label{eq:Linfstr-egal-Linfeq2}
\ln \frac{1}{\nu\bigl\{m(\nu)\bigr\}} =
\Linfeq\bigl(m(\nu),\nu\bigr) \leq \Linfstr\bigl(m(\nu),\nu\bigr) = +\infty\,.
\end{equation}
Analogous results hold for $\Lsupstr(\,\cdot\,,\nu)$, $\Lsupeq(\,\cdot\,,\nu)$, and $M(\nu)$.
\end{lemma}

\begin{proof}
To prove~\eqref{eq:Linfstr-egal-Linfeq1}, we first identify the interior of the domain of $\Linfstr$.

Distributions $\zeta$ such that $\Ed(\zeta) < m(\nu)$ cannot be absolutely continuous with
respect to $\nu$; otherwise, they would also give a null probability to values strictly smaller than $m(\nu)$,
which contradicts the assumption $\Ed(\zeta) < m(\nu)$. Hence $\KL(\zeta,\nu) = +\infty$ for these distributions. It follows that
$\Linfstr(\mu,\nu) = \Linfeq(\mu,\nu) = +\infty$ for $\mu < m(\nu)$; we note in passing that we also have
$\Linfstr\bigl(m(\nu),\nu\bigr) = +\infty$.

For $\mu > m(\nu)$, we take $\epsilon > 0$ with $m(\nu)+\epsilon < \mu$ and have, by definition of the support of a measure, that
$\bigl[m(\nu),m(\nu)+\epsilon]$ has a positive $\nu$--measure denoted by $\kappa$. The distribution $\zeta$ given by $\nu$ conditioned
to the interval $\bigl[m(\nu),m(\nu)+\epsilon]$ is absolutely continuous with respect to~$\nu$,
with density $\d\zeta/\d\nu = 1/\kappa$ on $\bigl[m(\nu),m(\nu)+\epsilon\bigr]$, and $0$ elsewhere; therefore, $\KL(\zeta,\nu) = \ln(1/\kappa) < +\infty$
and $\Linfstr(\mu,\nu) < +\infty$.

The interior of the domain of $\mu \mapsto \Linfstr(\mu,\nu)$ is therefore $\bigl( m(\nu), \, +\infty\bigr)$,
and we recall that $\Linfstr(\,\cdot\,,\nu)$ is continuous on this interval.
We fix some $\mu > m(\nu)$. For all $\epsilon > 0$,
by the very definitions of all quantities as infima of nested sets, we have
\[
\Linfstr(\mu-\epsilon,\nu) \leq \Linfeq(\mu,\nu) \leq \Linfstr(\mu,\nu)\,.
\]
Letting $\epsilon \to 0$, we get, by a sandwich argument, that $\Linfeq(\mu,\nu) = \Linfstr(\mu,\nu)$.
This concludes the proof of~\eqref{eq:Linfstr-egal-Linfeq1}.

We turn our attention to~\eqref{eq:Linfstr-egal-Linfeq2}.
We already showed above that $\Linfstr\bigl(m(\nu),\nu\bigr) = +\infty$. Now, to compute
$\Linfeq(\mu,\nu)$, we wonder which are the distributions $\zeta$ that
are absolutely continuous with respect to $\nu$, and thus,
give a null probability to values strictly smaller than $m(\nu)$,
and are also such that $\Ed(\zeta) \leq m(\nu)$: at most one
such distribution exists, the Dirac mass at $m(\nu)$, denoted by $\delta_{m(\nu)}$. We then distinguish the cases
$\nu\bigl\{m(\nu)\bigr\} > 0$ and $\nu\bigl\{m(\nu)\bigr\} = 0$
to establish, respectively, the equalities
\[
\smash{
\Linfeq\bigl(m(\nu),\nu\bigr) =
\KL(\delta_{m(\nu)},\nu) = \ln \frac{1}{\nu\bigl\{m(\nu)\bigr\}}
\quad \mbox{and} \quad \Linfeq\bigl(m(\nu),\nu\bigr) = + \infty = \ln \frac{1}{\nu\bigl\{m(\nu)\bigr\}}\,.
}
\]
In both cases, the first equality in \eqref{eq:Linfstr-egal-Linfeq2} is proved, which concludes the proof.
\end{proof}

We also have the following result, which is the most important and useful one,
as it discussed the quantity that appears in the upper bounds
on the average log-probability of misidentification of the optimal arm;
see Corollary~\ref{cor:SR} together with Lemma~\ref{lem:duality}.

\begin{restatable}{lem}{sumLinf}
Let $\nu, \nu' \in \cDall$ with $\mu = \Ed(\nu) > \Ed(\nu') = \mu'$. Then
\[
\inf_{x \in [\mu', \mu]} \Linfeq(x, \nu) + \Lsupeq(x, \nu') =
\inf_{x \in [\mu', \mu]} \Linfstr(x, \nu) + \Lsupstr(x, \nu')
\]
if and only if
either $m(\nu) \neq M(\nu')$ or $\nu\bigl\{m(\nu)\bigr\} \times \nu'\bigl\{M(\nu')\bigr\} = 0$.
\end{restatable}

\begin{remark} In other words, the only case for which the two infima differ is when
$m(\nu) = M(\nu')$, i.e., the upper end of the support of $\nu'$ equals the lower end
of the support of $\nu$, and both $\nu$ and $\nu'$ admit this common value as an atom.
\end{remark}

\begin{proof}
The first lines of the proof of Lemma~\ref{lm:Linfstr-egal-Linfeq} show
that $\Linfeq(x, \nu) = \Linfstr(x, \nu) = +\infty$ for $x < m(\nu)$. We can symmetrically
show that $\Lsupeq(x, \nu') = \Lsupstr(x, \nu') = +\infty$ for $x > M(\nu')$.
Therefore, $\Linfeq(x, \nu) + \Lsupeq(x, \nu')$ and $\Linfstr(x, \nu) + \Lsupstr(x, \nu')$ are infinite whenever $x$ lies outside of $\bigl[m(\nu), M(\nu')\bigr]$. This implies that
\begin{align*}
&&\inf_{x \in [\mu', \mu]} \Linfeq(x, \nu) + \Lsupeq(x, \nu') &= \inf_{x \in [\mu', \mu] \cap [m(\nu), M(\nu')]} \Linfeq(x, \nu) + \Lsupeq(x, \nu') \\
\text{and }	&& \inf_{x \in [\mu', \mu]} \Linfstr(x, \nu) + \Lsupstr(x, \nu') &=  \inf_{x \in [\mu', \mu] \cap [m(\nu), M(\nu')]} \Linfstr(x, \nu) + \Lsupstr(x, \nu') \,.
\end{align*}
We now split the analysis according to how large the interval $\cI$ is, where
\[
\mathcal{I} = [\mu', \mu] \cap \bigl[ m(\nu), M(\nu') \bigr]
= \Bigl[ \max\bigl\{\mu',m(\nu)\bigr\}, \, \min\bigl\{\mu,M(\nu')\bigr\} \Bigr]\,.
\]

\emph{Case~1: $\cI$ is empty.}~~In that case, the two infima
are over an empty set and both equal $+\infty$.
\smallskip

\emph{Case~2: $\cI$ has a non-empty interior.}~~When $a \ne b$, the
infimum of a convex function over a closed interval $[a,b]$
equals the infimum over $(a,b)$, whether the function takes finite
or infinite values at $a$ and $b$.
Now, the interior of $\cI = [a,b]$ equals
\[
(a,b) = \Bigl( \max\bigl\{\mu',m(\nu)\bigr\}, \, \min\bigl\{\mu,M(\nu')\bigr\} \Bigr)
= (\mu', \mu) \cap \bigl( m(\nu), M(\nu') \bigr)
\]
and does not contain neither $m(\nu)$ nor $M(\nu')$.
By Lemma~\ref{lm:Linfstr-egal-Linfeq}, the functions
$\Linfstr(\,\cdot\,,\nu)$ and $\Linfeq(\,\cdot\,,\nu)$ coincide
on $\R \setminus \bigr\{ m(\nu) \bigr\}$. It may be similarly shown that
$\Lsupstr(\,\cdot\,,\nu')$ and $\Lsupeq(\,\cdot\,,\nu')$ coincide
on $\R \setminus \bigr\{ M(\nu') \bigr\}$. In particular,
the functions $\Linfeq(\,\cdot\,, \nu) + \Lsupeq(\,\cdot\,, \nu')$
and $\Linfstr(\,\cdot\,, \nu) + \Lsupstr(\,\cdot\,, \nu')$ coincide on
the interior of $\cI$.
Their infima over the interior of $\cI$, which, by convexity, are equal to the infima over $\cI$,
are therefore equal.
\smallskip

\emph{Case~3: $\cI$ is a singleton.}~~This case arises if and only
if $m(\nu) = M(\nu')$, as by definition, $m(\nu) \leq \mu$ and $M(\nu') \geq \mu'$.
We then have
$\cI = \bigl\{ m(\nu) \bigr\} = \bigl\{ M(\nu') \bigr\}$,
and both infima are equal to the values of the sums at $m(\nu) = M(\nu')$.
By Lemma~\ref{lm:Linfstr-egal-Linfeq} and by symmetric results for $\Lsupstr$
and $\Lsupeq$, on the one hand,
\[
\Linfstr\bigl(m(\nu), \nu\bigr) = \Lsupstr\bigl(M(\nu'), \nu'\bigr) = +\infty \,,
\]
and on the other hand,
\[
\Linfeq\bigl(m(\nu), \nu\bigr) + \Lsupeq\bigl(M(\nu'), \nu'\bigr) =
\ln \frac{1}{\nu\bigl\{m(\nu)\bigr\}} + \ln \frac{1}{\nu'\bigl\{M(\nu')\bigr\}} \,.
\]
We get the desired equality if and only if either $\nu\bigl\{m(\nu)\bigr\} = 0$
or $\nu\bigl\{M(\nu')\bigr\} = 0$.
\end{proof}

\section{General analysis of successive-rejects in terms of $\Phi$}
\label{sec:app-SR}

This appendix is devoted to the technical elements omitted in the general analysis
of the successive-rejects strategy presented in Section~\ref{sec:SR-gal}.

\subsection{The Cram{\'e}r-Chernoff bound}
\label{sec:app-SR-CR}

In this section, we recall the statement of the highly classical Cram{\'e}r-Chernoff bound:
with the notation introduced in Section~\ref{sec:SR}, for an $N$--sample $X_1,\ldots,X_N$, distributed according to $\nu$
and of average denoted by $\ol{X}_N$,
\begin{align}
&\forall x \leq \Ed(\nu), \qquad \P\bigl( \ol{X}_N \leq x \bigr) \leq \exp \bigl( - N \, \phi^\star_{\nu}(x) \bigr)\,, \label{eq:CRB1} \\	
\text{and} \qquad  &\forall x \geq \Ed(\nu), \qquad \P\bigl( \ol{X}_N \geq x \bigr) \leq \exp \bigl( - N \, \phi^\star_{\nu}(x) \bigr)\,. \label{eq:CRB2}
\end{align}
Such a classical result would in principle not require to be proved here.
However, it turns out that we will re-use parts of this proof in later proofs, like the
application~\ref{eq:Jensen-phi} of Jensen's inequality or the variations of
$\phi^\star_\nu$ discussed at the end of this section. This is why, despite all,
we now prove~\eqref{eq:CRB1}--\eqref{eq:CRB2}. \\

\begin{proof}
For all $\lambda < 0$, by Markov's inequality first and then by independence,
\begin{align*}
\P\bigl( \ol{X}_N \leq x \bigr) =
\P\Bigl( \e^{\lambda \ol{X}_N} \geq \e^{\lambda x} \Bigr)
& \leq \e^{- \lambda x} \, \E\Bigl[ \e^{\lambda \ol{X}_N} \Bigr]
=  \e^{- \lambda x} \, \Bigl( \E\bigl[ \e^{\lambda X_1 / N} \bigr] \Bigr)^N \\
& = \exp \bigl( - \lambda x + N \, \phi_{\nu}(\lambda/N) \bigr)
= \exp \Bigl( - N \bigl( \lambda' x - \phi_{\nu}(\lambda') \bigr) \Bigr)\,,
\end{align*}
where $\lambda' = \lambda/N$. The bound also holds
for $\lambda = \lambda' = 0$ given that $\phi_{\nu}(0) = 0$.
Optimizing over $\lambda \leq 0$ (or, equivalently,
over $\lambda' \leq 0$), we proved so far
\[
\P\bigl( \ol{X}_N \leq x \bigr) \leq
\exp \biggl( - N \, \sup_{\lambda \leq 0} \bigl\{ \lambda x - \phi_{\nu}(\lambda) \bigr\} \biggr)\,.
\]
Now, by Jensen's inequality,
\begin{equation}
\label{eq:Jensen-phi}
\forall \lambda \in \R, \qquad
\phi_{\nu}(\lambda) = \ln \E\bigl[\e^{\lambda X}\bigr] \geq \lambda\,\E[X] = \lambda \, \Ed(\nu)\,;
\end{equation}
therefore, for $x \leq \Ed(\nu)$,
\[
\forall \lambda \geq 0, \qquad
\lambda x - \phi_{\nu}(\lambda) \leq \lambda \bigl( x - \Ed(\nu) \bigr) \leq 0\,.
\]
In particular,
\begin{equation}
\label{eq:phistar-cr-decr}
0 = - \phi_{\nu}(0) \leq \sup_{\lambda \leq 0} \bigl\{ \lambda x - \phi_{\nu}(\lambda) \bigr\}
= \sup_{\lambda \in \R} \bigl\{ \lambda x - \phi_{\nu}(\lambda) \bigr\} \eqdef \phi^\star_\nu(x)\,.
\end{equation}
This concludes the proof of~\eqref{eq:CRB1}.
The bound~\eqref{eq:CRB2} follows by symmetry.
\end{proof}

We also note, in passing, that Jensen's inequality entails, for $x = \Ed(\nu)$, that
\[
\forall \lambda \in \R, \qquad
\lambda \Ed(\nu) - \phi_{\nu}(\lambda) \leq \lambda \bigl( \Ed(\nu) - \Ed(\nu) \bigr) = 0\,,
\]
thus showing that $\phi^\star_\nu\bigl( \Ed(\nu) \bigr) = 0$.
The property~\eqref{eq:phistar-cr-decr} and its counterpart for $x \geq \Ed(\nu)$ and $\lambda \geq 0$
actually show that $\phi^\star_\nu$ is non-increasing on $\bigl( -\infty, \, \Ed(\nu)\bigr]$
and non-decreasing on $\bigl[ \Ed(\nu), \, +\infty\bigr)$.

\subsection{Proof of Lemma \ref{lem:devphistar}}
\label{sec:lmdevphistar}

We first restate the lemma, for the convenience of the reader.

\grille*

\begin{proof}
The proof consists in two parts. We first show that for any finite grid $\cG = \{ g_2,\ldots,g_{G-1}\}$ in $(\mu',\mu)$, to which we add
the points $g_1 = \mu'$ and $g_{G} = \mu$, we have
\begin{equation}
\label{eq:devphistar1}
\limsup_{N \to +\infty} \frac{1}{N} \ln \P \bigl( \overline{X}_N \leq \overline{Y}_N \bigr)
\leq - \min \Bigl\{ \phi^\star_{\nu}(\mu'), \,\,
\min_{2 \leq j \leq G-1} \bigl\{ \phi^\star_{\nu'}(g_{j-1}) + \phi^\star_{\nu}(g_j) \bigr\}, \,\,
\phi^\star_{\nu'}(\mu) \Bigr\}\,.
\end{equation}
Indeed, by identifying, when $\ol{X}_N$ and $\ol{Y}_N$ belong to $[\mu',\mu]$,
in which interval $[g_{j-1},g_j]$ lies $\ol{X}_N$, we note that
\[
\bigl\{ \ov{X}_N \leq \ov{Y}_N \bigr\}
\quad \subseteq \quad
\bigl\{\ov{X}_N \leq \mu' \bigr\} ~\cup~ \bigl\{\ov{Y}_N \geq \mu \bigr\}
~\cup~ \bigcup_{j=2}^{G-1} \bigl\{ \ov{Y}_N \geq g_{j-1} ~\mbox{and}~ \ov{X}_N \leq g_j \bigr\} \, .
\]
First,
by independence and by the Cram{\'e}r-Chernoff inequalities~\eqref{eq:CRB1} and~\eqref{eq:CRB2},
\[
\P\bigl( \ov{Y}_N \geq g_{j-1} ~\mbox{and}~ \ov{X}_N \leq g_j \bigr)
= \P \bigl( \ov{Y}_N \geq g_{j-1} \bigr) \,\, \P\bigl( \ov{X}_N \leq g_j \bigr)
\leq \exp \Bigl( - N \bigl( \phi^\star_{\nu'}(g_{j-1}) + \phi^\star_{\nu}(g_j) \bigr) \Bigr)\, .
\]
Second, again by the Cram{\'e}r-Chernoff inequalities,
\[
\P \bigl( \ov{X}_N \leq \mu' \bigr) \leq \exp \bigl( - N \, \phi^\star_{\nu}(\mu') \bigr)
\qquad \mbox{and} \qquad
\P \bigl( \ov{Y}_N \geq \mu \bigr) \leq \exp \bigl( - N \, \phi^\star_{\nu'}(\mu) \bigr) \, .
\]
By a union bound,
\[
\smash{\P \bigl( \overline{X}_N \leq \overline{Y}_N \bigr) \leq
  \exp \bigl( - N \, \phi^\star_{\nu}(\mu') \bigr)
+ \exp \bigl( - N \, \phi^\star_{\nu'}(\mu) \bigr)
+ \sum_{j=2}^{G-1} \exp \Bigl( - N \bigl( \phi^\star_{\nu'}(g_{j-1}) + \phi^\star_{\nu}(g_j) \bigr) \Bigr)\,.}
\]
The stated bound~\eqref{eq:devphistar1} follows by identifying the (finitely many) terms with the smallest
rate in the exponent.

\smallskip
In the second part of the proof, we note that the bound~\eqref{eq:devphistar1} holds for any finite grid
in $(\mu',\mu)$, and we consider a sequence
\[
\cG^{(n)} = \Bigl\{ g_2^{(n)}, \, \ldots,\,  g_{G_n-1}^{(n)} \Bigr\}
\]
of such finite grids. In particular,
\begin{align*}
\limsup_{N \to +\infty} \frac{1}{N} \ln \P \bigl( \overline{X}_N \leq \overline{Y}_N \bigr)
\leq - \min \Bigl\{ \phi^\star_{\nu}(\mu'), \,\,
\max_{n \geq 1} \ & S_n, \,\,
\phi^\star_{\nu'}(\mu) \Bigr\}\,, \\
\mbox{where} \qquad
& S_n \eqdef \min_{2 \leq j \leq G_n-1} \Bigl\{ \phi^\star_{\nu'}\bigl(g_{j-1}^{(n)}\bigr) + \phi^\star_{\nu}\bigl(g_j^{(n)}\bigr) \Bigr\} \, .
\end{align*}
To obtain the claimed bound,
given that (see the end of Appendix~\ref{sec:app-SR-CR})
\[
\phi^\star_{\nu}(\mu) = 0 = \phi^\star_{\nu'}(\mu')\,,
\]
it suffices to show that
\[
\max_{n \geq 1} S_n \geq \inf_{x \in [\mu', \mu]} \bigl\{ \phi^\star_{\nu'}(x) + \phi^\star_{\nu}(x) \bigr\}\,.
\]
To that end, we assume that the steps $\epsilon_n$ of the grids $\cG^{(n)}$, which are defined as
\[
\epsilon_n \eqdef \max_{2 \leq j \leq G_n} \Bigl|g_j^{(n)} - g_{j-1}^{(n)} \Bigr| \, ,
\]
vanish asymptotically, i.e., $\epsilon_n \to 0$. For each grid $\cG^{(n)}$, we denote by
$x^\star_n \in (\mu',\mu)$ the argument of the minimum in the definition of $S_n$.
As a consequence, for each $n \geq 1$,
\[
S_n = \phi^\star_{\nu'}(x^\star_n - \epsilon^\star_n) + \phi^\star_{\nu}(x^\star_n) \,,
\]
for some $0 < \epsilon^\star_n \leq \epsilon_n$.
The quantity $x^\star_n - \epsilon^\star_n$ denotes the point in the grid that is right before
$x^\star_n$, and it belongs to $[\mu',\mu)$. We note that we also have $\epsilon^\star_n \to 0$.
In the compact interval $[\mu',\mu]$,
the Bolzano-Weierstrass theorem (see, e.g., \citealp[Section~3.4]{BZref}) ensures the existence of a converging subsequence:
there exists $x^\star_\infty \in [\mu',\mu]$ and a sequence $(n_k)_{k \geq 1}$ of integers such that
\[
x^\star_{n_k} \underset{k \to +\infty}{\longrightarrow} x^\star_\infty\,,
\qquad \mbox{which also entails} \qquad
x^\star_{n_k} - \epsilon^\star_{n_k} \underset{k \to +\infty}{\longrightarrow} x^\star_\infty\,.
\]
Now, the functions $\phi^\star_{\nu}$, respectively, $\phi^\star_{\nu'}$, are lower semi-continuous, as the suprema
over $\lambda \in \R$ of the continuous functions $x \mapsto \lambda x - \varphi_\nu(\lambda)$, respectively, $x \mapsto \lambda x - \varphi_{\nu'}(\lambda)$.
Therefore, by these lower semi-continuities,
\begin{align*}
\max_{n \geq 1} S_n
\geq \liminf_{k \to +\infty} \phi^\star_{\nu'}(x^\star_{n_k} - \epsilon^\star_{n_k}) + \phi^\star_{\nu}(x^\star_{n_k})
& \geq \phi^\star_{\nu'}(x^\star_\infty) + \phi^\star_{\nu}(x^\star_\infty) \\
& \geq \inf_{x \in [\mu',\mu]} \bigl\{ \phi^\star_{\nu'}(x) + \phi^\star_{\nu}(x) \bigr\}\,.
\end{align*}
This concludes the proof.
\end{proof}

\subsection{Proof of Theorem~\ref{th:SR}}
\label{sec:thSR}

The proof mimics the analysis by~\citet{ABM10}, the main modification being the substitution of
Hoeffding's inequality by the bound of Lemma~\ref{lem:devphistar}. \\

\begin{proof}
We recall that for $r \in \{1, \ldots, K-1\}$, we denoted by $N_r = \lfloor \ell_1 / K \rfloor + \ldots + \lfloor \ell_r/(K-r+1) \rfloor$
the total number of times an arm still considered in phase $r$, i.e., belonging to $S_{r-1}$, was pulled
in phases~$1$ to~$r$. For each arm $a$, we denote by $\ol{Y}_a^r$ the average of a $N_r$--sample
distributed according to $\nu_a$.
By optional skipping (see \citealp[Chapter III, Theorem 5.2, p. 145]{doob1953}, or
\citealp[Section~5.3]{CT88} for a more recent reference), we may assume, with no loss of generality, that
for each $r \in \{ 1,\ldots,K-1 \}$,
\begin{equation}
\label{eq:opt-skip}
\mbox{on the event} \ \{a \in S_{r-1} \}, \qquad \ol{X}_a^r = \ol{Y}_a^r\,.
\end{equation}

We fix a bandit problem $\unu$ with a unique optimal arm $a^\star(\unu)$.
The successive-rejects strategy fails if (and only) if it rejects $a^\star(\unu)$ in ones of the phases.
This corresponds to the event
\[
\bigl\{ I_T \ne a^\star(\unu) \bigr\}
= \bigcup_{r=1}^{K-1} \bigl\{ a_r = a^\star(\unu) \bigr\}
\subseteq \bigcup_{r=1}^{K-1} \Bigl\{ a^\star(\unu) \in S_{r-1} \ \ \mbox{and} \ \
\forall k \in S_{r-1}, \ \ \ol{X}^r_{a^\star(\unu)} \leq \ol{X}^r_{k} \Bigr\} \, .
\]
(We have an inclusion because ties are broken arbitrarily.)
By optional skipping~\eqref{eq:opt-skip},
\begin{multline*}
\bigcup_{r=1}^{K-1} \Bigl\{ a^\star(\unu) \in S_{r-1} \ \ \mbox{and} \ \
\forall k \in S_{r-1}, \ \ \ol{X}^r_{a^\star(\unu)} \leq \ol{X}^r_{k} \Bigr\} \\
= \bigcup_{r=1}^{K-1} \Bigl\{ a^\star(\unu) \in S_{r-1} \ \ \mbox{and} \ \
\forall k \in S_{r-1}, \ \ \ol{Y}^r_{a^\star(\unu)} \leq \ol{Y}^r_{k} \Bigr\}\,.
\end{multline*}
Recall that the set $S_{r-1}$ is a random set; dealing with it therefore requires some care.
On the event of interest, $S_{r-1}$ contains $K-r+1$ elements, among which $a^\star(\unu)$.
The set $\cA_r$ is of cardinality $r$ and does not contain $a^\star(\unu)$.
By the pigeonhole principle, $S_{r-1}$ thus necessarily contains one arm in~$\cA_r$. As a consequence,
for each phase $r \in \{1,\ldots,K-1\}$,
\[
\Bigl\{ a^\star(\unu) \in S_{r-1} \ \ \mbox{and} \ \
\forall k \in S_{r-1}, \ \ \ol{Y}^r_{a^\star(\unu)} \leq \ol{Y}^r_{k} \Bigr\}
\subseteq \bigcup_{k \in \cA_r} \Bigl\{ \ol{Y}^r_{a^\star(\unu)} \leq \ol{Y}^r_{k} \Bigr\}\,.
\]
Summarizing the inclusions above, taking unions bounds, and upper bounding
the obtained sum in a crude way, we proved so far
\[
\P \bigl( I_T \ne a^\star(\unu) \bigr)
\leq \sum_{r=1}^{K-1} \sum_{k \in \cA_r} \P\Bigl( \ol{Y}^r_{a^\star(\unu)} \leq \ol{Y}^r_{k} \Bigr)
\leq K^2 \max_{1 \leq r \leq K-1} \,\, \max_{k \in \cA_r} \P\Bigl( \ol{Y}^r_{a^\star(\unu)} \leq \ol{Y}^r_{k} \Bigr)\,,
\]
or equivalently,
\begin{align*}
\frac{1}{T} \ln \P \bigl( I_T \ne a^\star(\unu) \bigr)
& \leq \frac{2}{T} \ln K + \max_{1 \leq r \leq K-1} \,\, \max_{k \in \cA_r}
\frac{1}{T} \ln \P\Bigl( \ol{Y}^r_{a^\star(\unu)} \leq \ol{Y}^r_{k} \Bigr) \\
& = \frac{2}{T} \ln K + \max_{1 \leq r \leq K-1} \,\, \max_{k \in \cA_r} \, \frac{N_r}{T} \frac{1}{N_r} \ln \P\Bigl( \ol{Y}^r_{a^\star(\unu)} \leq \ol{Y}^r_{k} \Bigr)\,.
\end{align*}
As $N_r/T \to \gamma_r > 0$ as $T \to +\infty$, we may apply Lemma~\ref{lem:devphistar},
together with an exchange between the $\limsup$ and the maximum over a finite number
of quantities. We obtain
\begin{multline*}
\limsup_{T \to +\infty} \frac{1}{T} \ln \P \bigl( I_T \ne a^\star(\unu) \bigr)
\leq \max_{1 \leq r \leq K-1} \max_{k \in \cA_r} \, \biggl\{ \gamma_r \Bigl( - \Phi\bigl( \nu_{k}, \nu^\star \bigr) \Bigr) \biggl\} \\
= - \min_{1 \leq r \leq K-1} \left\{ \gamma_r  \min_{k \in \cA_r} \Phi\bigl( \nu_{k}, \nu^\star \bigr)
\right\}.
\end{multline*}
This concludes the proof.
\end{proof}

\subsection{Proof of Corollary~\ref{cor:SR} and of the bound~\eqref{eq:phistarHbd} on $\Phi$}
\label{sec:cor-SR}

In this final subsection, we provide two series of proofs:
first, a proof of Corollary~\ref{cor:SR}; and then a proof of the bound
$\Phi\bigl( \nu_k, \nu^\star \bigr) \geq \Delta_k^2$ stated as~\eqref{eq:phistarHbd}. \\

\begin{proof} \textbf{of Corollary~\ref{cor:SR}.}~~To
apply Theorem~\ref{th:SR}, we need only to show that the phase lengths of~\eqref{eq:reglengths-ABM10} are such that $N_r/T$ converges to a positive value, and to identify this limit value~$\gamma_r$.
As $N_1 = \lfloor \ell_1/K \rfloor$, where $\ell_1 = T/\bln K$, we immediately have
$N_1/T \to \gamma_1 = 1 / \bigl( K \bln K \bigr) > 0$.
For $r \in \{2, \ldots, K-1\}$,
\begin{align*}
\frac{N_r}{T}
& = \sum_{p=1}^r \frac{1}{T} \biggl\lfloor \frac{\ell_p}{K} \biggr\rfloor
= \frac{1}{T} \Biggl( \biggl\lfloor \frac{T}{K \bln K} \biggr\rfloor + \sum_{p=2}^r \biggl\lfloor \frac{T}{(K-p+1)(K-p+2)\bln K} \biggr\rfloor \Biggr) \\
& \lriT \gamma_r \eqdef
\frac{1}{\bln K} \biggl( \frac{1}{K} + \sum_{p=2}^r \frac{1}{K-p+1} - \frac{1}{K-p+2} \biggr)
= \frac{1}{(K-r+1) \bln K}\,.
\end{align*}
The bound of Theorem~\ref{th:SR} reads:
\[
\limsup_{T \to +\infty} \frac{1}{T} \ln \P \bigl( I_T \ne a^\star(\unu) \bigr)
\leq - \frac{1}{\bln K} \min_{1 \leq r \leq K-1} \left\{ \frac{1}{K-r+1} \min_{k \in \cA_r} \Phi\bigl( \nu_{k}, \nu^\star \bigr) \right\}.
\]
It implies, in terms of lower bounds $f(\nu_k, \nu^\star) \leq \Phi\bigl( \nu_{k}, \nu^\star \bigr)$,
\begin{equation}	\label{eq:cor3_final}
\limsup_{T \to +\infty} \frac{1}{T} \ln \P \bigl( I_T \ne a^\star(\unu) \bigr)
\leq - \frac{1}{\bln K} \min_{1 \leq r \leq K-1} \left\{ \frac{1}{K-r+1} \min_{k \in \cA_r} f\bigl( \nu_{k}, \nu^\star \bigr) \right\}.
\end{equation}
The permutation $\sigma$ in~\eqref{eq:orderingU} and the sets $\cA_r = \{\sigma_{K-r+1}, \ldots, \sigma_{K}\}$ were exactly picked,
for each $r \in \{1, \ldots, K-1\}$, to minimize
\[
\min_{k \in \cB_r} f\bigl( \nu_{k}, \nu^\star \bigr)
\]
over sets $\cB_r$ abiding by the indicated constraints:
being of cardinal $r$ and not containing the optimal arm $a^\star(\unu) = \sigma_1$.
We get
\[
\min_{k \in \cA_r} f\bigl( \nu_{k}, \nu^\star \bigr) = \min_{K-r+1 \leq k \leq K} f\bigl( \nu_{\sigma_k}, \nu^\star \bigr) = f\bigl( \nu_{\sigma_{K-r+1}}, \nu^\star \bigr) \,,
\]
which, together with~\eqref{eq:cor3_final}, yields the stated bound, up to replacing $K-r+1$ with $r \in \{1, \ldots, K-1\}$ by $k \in \{2, \ldots, K\}$:
\[
- \frac{1}{\bln K} \min_{1 \leq r \leq K-1} \left\{ \frac{1}{K-r+1} f\bigl( \nu_{\sigma_{K-r+1}}, \nu^\star \bigr) \right\}
= - \frac{1}{\bln K} \min_{2 \leq k \leq K} \left\{ \frac{1}{k} f\bigl( \nu_{\sigma_{k}}, \nu^\star \bigr) \right\}.
\vspace{-.75cm}
\]

\end{proof}

We now move to the proof of the bound~\eqref{eq:phistarHbd} on $\Phi$, when the model is $\cD = \cDall$;
we restate this bound here for the convenience of the reader:
\[
\boundphiHbd\,.
\]
For the ease of exposition,
the path followed in Section~\ref{sec:overview} to show that $\Phi\bigl( \nu_k, \nu^\star \bigr) \geq \Delta_k^2$
was to first note that $\Phi = \chern$ when $\cD = \cDall$ (see Lemma~\ref{lem:duality})
and then use Pinsker's inequality~\eqref{eq:chern-gaps}.
We provide here a slightly more direct but equivalent approach, based on Hoeffding's inequality. \\

\begin{proof} \textbf{of the bound~\eqref{eq:phistarHbd} on $\Phi$.}~~When
$\nu \in \cDall$, Hoeffding's inequality exactly states that
\begin{align*}
& \forall \lambda \in \R, \quad \phi_\nu(\lambda) \leq \lambda \Ed(\nu) + \frac{\lambda^2}{8}\,, \\
\mbox{so that} \qquad &
\forall x \in \R, \quad \phi^\star_\nu(x) \geq
\sup_{\lambda \in \R} \biggl\{ \lambda \bigl( x - \Ed(\nu) \bigr) - \frac{\lambda^2}{8} \biggr\}
= 2 \bigl( x - \Ed(\nu) \bigr)^2\,.
\end{align*}
This corresponds to the first part of~\eqref{eq:phistarHbd}.

For its second part, we consider a pair $\nu,\nu'$ of distributions in $\cDall$,
we set any $x \in [\Ed(\nu'), \Ed(\nu)]$, and we apply twice the bound of the first part
to get
\[
\phi^\star_{\nu'}(x) + \phi^\star_{\nu}(x) \geq 2 \bigl( x - \Ed(\nu') \bigr)^2
+ 2 \bigl( x - \Ed(\nu) \bigr)^2\,.
\]
From the definition of~$\Phi$, it follows that
\[
\Phi(\nu',\nu) \geq
\inf_{x \in [\Ed(\nu'), \Ed(\nu)]} \Bigl\{ 2\bigl( x - \Ed(\nu') \bigr)^2
+ 2\bigl( x - \Ed(\nu) \bigr)^2 \Bigr\} = \bigl( \Ed(\nu') - \Ed(\nu) \bigr)^2\,.
\]
This corresponds to the second part of~\eqref{eq:phistarHbd}.
\end{proof}

\section{Proofs and details for Section~\ref{sec:Phi=L}: Rewriting of $\Phi$ as $\chern$}	
\label{app:phi=L}

We use the notation of Sections~\ref{sec:Linf} and~\ref{sec:SR} and discuss conditions on models guaranteeing that $\Phi = \chern$, i.e., that~\eqref{eq:Phi=L} holds. We do so for $\cD = \cDall$ in Appendix~\ref{sec:appcDall}
and for canonical one-parameter exponential families in Appendix~\ref{app:cDexp}.
Based on these two examples, we provide a set of conditions for general models, in Appendix~\ref{app:Phi=L_condtion}.
A building block of these results is that for all these models $\cD$, the functions $\Linfeq(\,\cdot\,, \nu)$ and $\Lsupeq(\,\cdot\,, \nu)$ dominate
the Fenchel-Legendre transform $\phi^\star_\nu$ defined in~\eqref{eq:defphistar}; we prove this in Appendix~\ref{app:C:deb}.

All proofs of this section are immediate adaptations of a rather standard result, stated, among others, but in a slightly different form (and for the model $\cD$
of all real-valued distributions with a first moment), by \citet[Exercise~4.13]{BLM16}.

\begin{remark}
\label{rk:Kinf-regret}
This rewriting of $\Linfeq(\,\cdot\,, \nu)$ or $\Lsupeq(\,\cdot\,, \nu)$
as $\phi^\star_\nu$ claimed, e.g., by Lemma~\ref{lem:duality},
can be seen as a counterpart to a similar rewriting of the $\cK_{\inf}$ as the supremum of a function of~$\lambda \in [0,1]$.
More precisely, we recall (see Remark~\ref{rk:Kinf})
that the $\cK_{\inf}$ function is defined, for $\nu \in \cDall$ and $x \in [0, 1]$, as
\[
\cK_{\inf}(\nu, x) = \inf \bigl\{ \KL(\nu, \zeta) : \ \zeta \in \cDall \ \ \mbox{\rm s.t.} \ \ \Ed(\zeta) > x \bigr\}\,,
\]
and \citet[Theorem 2]{honda_non-asymptotic_2015}---see also~\citealp[Lemma 18]{KLUCBS}---show that
\[
\cK_{\inf}(\nu, x) = \sup_{0 \leq \lambda \leq 1} \E\Biggl[ \ln \biggl( 1-\lambda \frac{X - x}{1-x}\biggr) \Biggr] \,,
\]
where $X$ is a random variable distributed according to $\nu$.
In both cases, for $\Linfeq(\,\cdot\,, \nu)$ or $\Lsupeq(\,\cdot\,, \nu)$,
and for $\cK_{\inf}$, being able to rewrite the infimum of Kullback-Leibler divergences as a supremum is
not unexpected: a given Kullback-Leibler divergence can be formulated as a supremum, see~\eqref{eq:KLmax},
and equalities between $\inf \sup$ and $\sup \inf$ holds under suitable assumptions (provided, e.g.,
by Sion's lemma).
\end{remark}

\subsection{$\Linfeq(\,\cdot\,, \nu)$ and $\Lsupeq(\,\cdot\,, \nu)$ dominate $\phi^\star_\nu$}
\label{app:C:deb}

This domination is a consequence of a variational formula~\eqref{eq:KLmax} for the Kullback-Leibler divergences.

\begin{lemma} \label{lem:variational}
For all models $\cD$ containing distributions with finite first moments, for all distributions $\nu \in \cD$,
\[
\forall x \leq \Ed(\nu), \quad \phi^\star_\nu(x) \leq \Linfeq(x,\nu)
\qquad \mbox{and} \qquad
\forall x \geq \Ed(\nu), \quad \phi^\star_\nu(x) \leq \Lsupeq(x,\nu) \,.
\]
\end{lemma}

\begin{proof}
We rely on a key variational formula for the Kullback-Leibler divergence, see \citet[Corollary~4.15]{BLM16}:
for all distributions $\nu,\,\nu'$ over~$\R$,
\begin{align}
\nonumber
\KL(\nu', \nu) & = \sup \Bigl\{ \E_{\nu'}[Y] - \ln \E_{\nu} \big[\e^{Y}\big] : \ \mbox{r.v.} \ Y \in \mathbb{L}^1(\nu') \ \
\mbox{s.t.} \ \ \E_{\nu}\big[\e^{Y}\big] < + \infty \Bigr\}\,, \\
\label{eq:KLmax}
& = \sup \Bigl\{ \E_{\nu'}[Y] - \ln \E_{\nu} \big[\e^{Y}\big] : \ \mbox{r.v.} \ Y \in \mathbb{L}^1(\nu') \Bigr\}\,,
\end{align}
where the supremum is over random variables $Y : \R \to \R$ with a finite first moment with respect to $\nu'$, and
where $\E_{\nu}$ and $\E_{\nu'}$ indicate that expectations are relative to $\nu$ and $\nu'$, respectively.
In particular, when $\nu$ and $\nu'$ lie in $\cD$, they admit finite first moments,
hence all random variables of the form $Y = \lambda \, \idR$ are $\nu'$--integrable, where
$\idR$ denotes the identity function over $\R$ and where
$\lambda \in \R$. We have $\E_{\nu'}[Y] = \lambda \, \Ed(\nu')$.
A consequence of~\eqref{eq:KLmax} and of the definition~\eqref{eq:defphistar} of $\phi^\star_\nu$ is therefore that
\begin{equation}
\label{eq:KLvarlambdaid}
\KL(\nu', \nu) \geq \sup_{\lambda \in \R} \Bigl\{ \lambda \, \Ed(\nu') - \ln \E_{\nu}\big[\e^{\lambda \idR}\big] \Bigr\} =
\phi^\star_\nu \bigl( \Ed(\nu') \bigr)\,.
\end{equation}
Using the variations of $\phi^\star_\nu$ indicated at the end of Appendix~\ref{sec:app-SR-CR}, we see
that
\[
\phi^\star_\nu \bigl( \Ed(\nu') \bigr) \geq \phi^\star_\nu(x) \quad \mbox{when} \ \ \Ed(\nu') \leq x \leq \Ed(\nu)
\quad \mbox{or} \quad \Ed(\nu') \geq x \geq \Ed(\nu)\,.
\]
Therefore, taking an infimum in~\eqref{eq:KLvarlambdaid} yields, when $x \leq \Ed(\nu)$,
\[
\Linfeq(x,\nu) = \inf \bigl\{ \KL(\nu', \nu) : \Ed(\nu') \leq x \bigr\} \geq \phi^\star_\nu(x)\,,
\]
and similarly for the other claimed inequality.
\end{proof}

\subsection{The case of $\cDall$}
\label{sec:appcDall}

In this section, we focus on the model $\cDall$ and prove that the inequalities of Lemma~\ref{lem:variational} are in fact equalities, as claimed by Lemma~\ref{lem:duality}, which we restate below. This yields, in particular, the target
equality~\eqref{eq:Phi=L}, as discussed after the statement of Lemma~\ref{lem:duality} in the main body of the article.

\duality*

The lemma holds for all $x \in \R$, that is, even outside of the $[0,1]$ interval,
though the proof reveals that when $x$ is smaller than the lower end $m(\nu)$ of the support of $\nu$,
we actually have $\phi^\star_\nu(x) = \Linfeq(x,\nu) = +\infty$. The counterpart statement
$\phi^\star_\nu(x) = \Lsupeq(x,\nu) = + \infty$ holds for $x$ larger than the
upper end $M(\nu)$ of the support of~$\nu$.
The pieces of notation $m(\nu)$ and $M(\nu)$ were formally defined in Appendix~\ref{sec:spec-cDall}. \\

\begin{proof}
Note first that by Lemma~\ref{lem:variational}, it suffices to prove that
\[
\forall x \leq \Ed(\nu), \quad \phi^\star_\nu(x) \geq \Linfeq(x,\nu)
\qquad \mbox{and} \qquad
\forall x \geq \Ed(\nu), \quad \phi^\star_\nu(x) \geq \Lsupeq(x,\nu) \,.
\]
We only deal with the first inequality, namely $\Linfeq(x,\nu) \leq \phi^\star_\nu(x)$ for $x \leq \Ed(\nu)$, as the other one
may be obtained by symmetric arguments.

In the case $x = \Ed(\nu)$, we have $\phi^\star_\nu\bigl( \Ed(\nu) \bigr) = 0$,
as stated at the end of Appendix~\ref{sec:app-SR-CR}, and $\Linfeq\bigl( \Ed(\nu), \nu \bigr) = 0$,
as can be seen by taking $\zeta = \nu$ in the infimum defining $\Linfeq$.
We therefore only consider $x < \Ed(\nu)$ in the sequel.
We will rely on the standard fact that, by Hölder's inequality,
the logarithmic moment-generating function
\[
\phi_{\nu} : \lambda \in \R \longmapsto \ln \E_{\nu}\big[\e^{\lambda \idunit}\big] \,,
\]
is convex, where $\idunit$ denotes the identity function on $[0, 1]$.
Also, by two applications of a standard theorem of differentiation under the integral, given that $\nu$
is supported by $[0,1]$, we have that $\phi_{\nu}$ is continuously differentiable over $\R$, with derivative
\[
\phi'_{\nu} : \lambda \in \R \longmapsto \frac{\E_{\nu}\big[\idunit \, \e^{\lambda \idunit}\big]}{\E_{\nu}\big[\e^{\lambda \idunit}\big]} \,.
\]
By convexity of $\phi_{\nu}$, this derivative is non-decreasing. Therefore, the limit of $\phi'_{\nu}$ at $-\infty$ exists; we denote it
by~$\ell$ and have that a priori $\ell \in \{-\infty\} \cup \R$.
We now prove that actually,
\begin{equation}
\label{eq:ell-mnu}
\ell \eqdef \lim_{\lambda \to -\infty} \phi'_{\nu}(\lambda) = m(\nu)\,.
\end{equation}
On the one hand, by definition of $m(\nu)$, we have $\idunit \geq m(\nu)$ $\nu$-a.s., which entails $\phi'_{\nu}(\lambda) \geq m(\nu)$ for all $\lambda \in \R$, and hence, $\ell \geq m(\nu)$.
On the other hand, as $\phi'_{\nu}$ is non-decreasing,
it is always larger than its limit $\ell$ at $-\infty$:
\begin{align}
\forall \lambda \in \R, \qquad
\phi'_{\nu}(\lambda) \geq \ell\,, \qquad & \mbox{thus,} \qquad
\E_{\nu}\Big[ \bigl(\idunit-\ell\bigr) \, \e^{\lambda \idunit}\Big] \geq 0\,, \label{eq:limit=m(nu)1}\\
& \mbox{or} \qquad
\E_{\nu}\Big[\bigl(\idunit-\ell\bigr) \, \e^{\lambda (\idunit-\ell)}\Big] \geq 0\,. \label{eq:limit=m(nu)2}
\end{align}
The last inequality and limit arguments as $\lambda \to -\infty$ impose that $\idunit-\ell \geq 0$ $\nu$-a.s., which in turn entails that $\ell \leq m(\nu)$.
This concludes the proof of~\eqref{eq:ell-mnu}.

The various properties exhibited above for $\phi_\nu$, including the fact that
the derivative $\phi'_{\nu}$ takes values in $\bigl[ m(\nu), +\infty \bigr)$, entail that the function
\[
\smash{\Lambda : \lambda \in \R \longmapsto \lambda x - \phi_{\nu}(\lambda)}
\]
is concave, continuously differentiable, with a non-increasing derivative $\Lambda'$ taking values in the interval $\bigl( -\infty, x-m(\nu) \bigr]$
and with limit $x-m(\nu)$ at $-\infty$.

We split the analysis of the case $x < \Ed(\nu)$ into three sub-cases,
depending on the respective positions of $x$ and $m(\nu)$, and recall that we want to show that
$\Linfeq(x,\nu) \leq \phi^\star_\nu(x)$. \smallskip

\emph{Case~1: $x > m(\nu)$.}~~By Jensen's inequality~\eqref{eq:Jensen-phi} and given that we consider $x < \Ed(\nu)$, the limit of $\Lambda$ at $+\infty$ equals $-\infty$. The limit of $\Lambda$ at $-\infty$ also equals $-\infty$, as the derivative $\Lambda'$ has limit $x-m(\nu) > 0$ at $-\infty$.
By concavity of $\Lambda$ and the fact that $\Lambda'$ is continuous, this implies the existence of some $\lambda^\star \in \R$ such that
\[
\Lambda'(\lambda^\star) = x - \phi_\nu'(\lambda^\star) = 0 \qquad \mbox{and} \qquad
\phi_\nu^\star(x) = \sup_{\lambda \in \R} \big\{ \Lambda(\lambda) \big\} = \Lambda(\lambda^\star) \,.
\]
Denoting by $\zeta_{\lambda^\star}$ the distribution absolutely continuous with respect to $\nu$ with density
\[
\frac{\d\zeta_{\lambda^\star}}{\d\nu} = \frac{\e^{\lambda^\star \idunit}}{\E_{\nu}\big[\e^{\lambda^\star \idunit}\big]}
= \e^{\lambda^\star \idunit - \phi_{\nu}(\lambda^\star)}\,,
\]
we have $\E_{\zeta_{\lambda^\star}} \!\bigl[ \idunit \bigr] = \Ed\bigl(\zeta_{\lambda^\star}\bigr) = \phi_\nu'(\lambda^\star) = x$. Therefore, by definition of $\Linfeq(x,\nu)$ and of the Kullback-Leibler divergence,
\[
\Linfeq(x,\nu) \leq \KL\bigl(\zeta_{\lambda^\star},\nu\bigr) = \E_{\zeta_{\lambda^\star}} \! \left[ \ln \frac{\d\zeta_{\lambda^\star}}{\d\nu} \right]
= \lambda^\star \, \E_{\zeta_{\lambda^\star}} \!\bigl[ \idunit \bigr] - \phi_{\nu}(\lambda^\star)
= \Lambda(\lambda^\star) = \phi^\star_\nu(x)\,.
\]

\emph{Case~2: $x = m(\nu)$.}~~In that case, $\Lambda' \to 0$ at $-\infty$ and $\Lambda'$ is non-increasing,
thus $\Lambda' \leq 0$ on $\R$ and $\Lambda$ is non-increasing on $\R$. Thus,
\[
\phi_\nu^\star\bigl(m(\nu)\bigr) = \sup_{\lambda \in \R} \big\{ \Lambda(\lambda) \big\} = \lim_{\lambda \to -\infty} \Lambda(\lambda) = \lim_{\lambda \to -\infty} - \ln \E_\nu \Bigl[ \e^{\lambda(\idunit-m(\nu))} \Bigr]\,.
\]
By monotone convergence
based on $\idunit - m(\nu) \geq 0$ $\nu$-a.s.,
\[
\lim_{\lambda \to -\infty} - \ln \E_\nu \Bigl[ \e^{\lambda(\idunit-m(\nu))} \Bigr] = - \ln \nu \bigl\{ m(\nu) \bigr\}\,,
\]
whether $\nu \bigl\{ m(\nu) \bigr\}$ is positive or null.
Moreover, Lemma~\ref{lm:Linfstr-egal-Linfeq} states that
\[
\Linfeq \bigl( m(\nu), \nu \bigr) = - \ln \nu \bigl\{ m(\nu) \bigr\}\,.
\]
We therefore have $\Linfeq(x,\nu) = \phi^\star_\nu(x)$ in this case. \smallskip

\emph{Case~3: $x < m(\nu)$.}~~In that case,
as $\Lambda' \to x-m(\nu) < 0$ at $-\infty$, we get that
$\Lambda\to +\infty$ at $-\infty$, thus $\phi_\nu^\star(x) = \sup \Lambda = +\infty$.
Now, no distribution $\zeta \in \cDall$ with $\Ed(\zeta) \leq x$, if some exists,
can be absolutely continuous with respect to $\nu$; indeed,
$x < m(\nu)$ imposes that $\zeta$ puts some
probability mass to the left of the support of $\nu$. Therefore, $\KL(\zeta,\nu) = +\infty$.
All in all, $\Linfeq( x, \nu)$ appears as the infimum of either an empty set
or of $+\infty$ values, so that $\Linfeq( x, \nu) = +\infty$.
In this case as well, $\Linfeq(x,\nu) = \phi^\star_\nu(x)$, both being equal to $+\infty$.
\end{proof}

\subsection{The case of canonical one-parameter exponential models $\cDexp$}	
\label{app:cDexp}

In this section, we show that the target equality~\eqref{eq:Phi=L} is satisfied by
so-called canonical one-parameter exponential families $\cDexp$. Before we do so, we recall
the definition and the properties of the latter.

\paragraph{Canonical one-parameter exponential families.}
We follow largely the exposition by \citet[Section~4]{AoS13};
more details, including the proofs of the stated properties may be found in the monograph by~\cite{LC98}.
A (regular) canonical one-parameter exponential family $\cDexp$ is a set of distributions $\nu_\theta$
indexed by $\theta \in \Theta$, all absolutely continuous with respect to some measure $\rho$ on $\R$,
with densities given by
\begin{equation}	
\label{eq:density_cDexp}
\frac{\d \nu_\theta}{\d \rho} = \exp\bigl(\theta \, \idR - b(\theta)\bigr) \,,
\end{equation}
for some smooth enough normalization function $b$. More precisely, $b$ is assumed to be twice differentiable.
We also assume that $\Theta$ is the natural parameter space, i.e., that $\Theta$ contains all possible parameters for $\rho$:
\[
\Theta = \Biggl\{ \theta \in \R \,:\, \int_{\R} \exp(\theta y) \,\d\rho(y) < + \infty \Biggr\} \,,
\]
and that $\Theta$ is an open interval (this latter fact is what regularity stands for).
A closed-form expression of $b$ is: for all $\theta \in \Theta$,
\begin{equation}
\label{eq:expression_b(theta)}
b(\theta) = \ln \int_\R e^{\theta y} \,\d\rho(y) \,.
\end{equation}
The derivative $b'$ of $b$ is a continuous function, by assumption, and it may be shown
that it is increasing, so that $b'$ is a one-to-one mapping with a continuous inverse $(b')^{-1}$.
In addition, it can be seen, by a differentiation under the integral sign, that
$\Ed(\nu_\theta) = b'(\theta)$ for all $\theta \in \Theta$.
Therefore, the distributions in $\cDexp$ may be rather parameterized by their expectations.
We denote by $\cM = b'(\Theta)$ the open interval of the expectations of distributions in $\cDexp$,
and let $\mu_-$ and $\mu_+$ be its lower and upper ends:
\[
\cM = (\mu_-,\mu_+)\,.
\]
For each $x \in \cM$, there exists a unique distribution in $\cDexp$ with expectation $x$, namely,
$\nu_{(b')^{-1}(x)}$.

\paragraph{Kullback-Leibler divergences for $\cDexp$.}
We may also parameterize the
Kullback-Leibler divergence function by the expectations: we define, for all $\theta_1, \theta_2 \in \Theta$,
\begin{equation}
\label{eq:d-KL}
d\bigl(\Ed(\nu_{\theta_1}), \Ed(\nu_{\theta_2})\bigr) \eqdef \KL(\nu_{\theta_1}, \nu_{\theta_2}) \,.
\end{equation}
This defines a divergence $d$ which is strictly convex and differentiable on the open set $\cM \times \cM$. In particular, $d$ is continuous, is such that $d(\mu, \mu') = 0$ if and only if $\mu = \mu'$, and, for all $\mu \in \cM$, both $d(\mu, \,\cdot\,)$ and $d(\,\cdot\,, \mu)$ are decreasing on $(\mu_-, \mu]$, and increasing on $[\mu, \mu^+)$.
In the following, we extend $d$ to $\R \times \R$ by $+\infty$ values outside of $\cM \times \cM$.

A direct application of the continuity and monotonicity properties of $d$ is that all functions $\Linfstr$, $\Linfeq$, $\Lsupstr$, $\Lsupeq$ coincide
with $d$ in the sense of the stated equalities~\eqref{eq:cont-d1} and~\eqref{eq:cont-d2}. Indeed and for instance, we have,
for $\nu \in \cDexp$ and $x \leq \Ed(\nu)$ with $x \in \cM$:
\[
\Linfstr(x, \nu) = \inf_{\mu < x} \bigl\{ d(\mu, \nu) \bigr\} = \lim_{\substack{\mu \to x \\ \mu < x}} d(\mu, \nu) = d(x, \nu) \,.
\]
When $x \notin \cM$, by the convention on the infimum of an empty set, $\Linfstr(x, \nu) = +\infty$,
while by our definition of $d$ outside $\cM \times \cM$, we also have $d(x, \nu) = +\infty$.
But as Lemma~\ref{lem:duality-dexp} below illustrates, we will only be interested on the behaviors on $\cM \times \cM$.

We now state a monotonicity property of the Chernoff-information-type quantity $L$ defined for exponential models in~\eqref{eq:Lchernoff_Dexp}.
This property was referred to in Example~\ref{ex:L-Ber}, when indicating that arms can be equivalently ranked
in descending expectations or ascending values of $L(\,\cdot\,, \mu^\star)$.

\begin{lemma} \label{lm:monotony_chernoff_Dexp}
Consider a canonical one-parameter exponential family $\cDexp$ and fix any $\mu \in \cM$.
Then $L(\,\cdot\,, \mu)$ is non-increasing on $(\mu_-, \mu]$.
\end{lemma}

\begin{proof}
Fix $\mu_- < \mu_2 \leq \mu_1 \leq \mu$. To get the
desired inequality $L(\mu_2, \mu) \geq L(\mu_1, \mu)$,
it suffices to show, by~\eqref{eq:Lchernoff_Dexp}, that
\begin{equation}	
\label{eq:proof_monotony_chernoff_Dexp}
\forall y \in [\mu_2, \mu], \qquad
d(y, \mu_2) + d(y, \mu) \geq \min_{x \in [\mu_1, \mu]} d(x, \mu_1) + d(x, \mu) \eqdef L(\mu_1,\mu) \,.
\end{equation}
We distinguish two cases. If $\mu_2 \leq \mu_1 \leq y \leq \mu$, then, since $d(y, \,\cdot\,)$ is increasing on $(\mu_-, y]$,
we have $d(y,\mu_2) \geq d(y,\mu_1)$, from which the inequality~\eqref{eq:proof_monotony_chernoff_Dexp} follows by
considering $x = y$. If $\mu_2 \leq y \leq \mu_1 \leq \mu$, then similarly $d(y, \mu) \geq d(\mu_1, \mu)$,
which yields
\[
\underbrace{d(y, \mu_2)}_{\geq 0} + d(y, \mu) \geq d(\mu_1, \mu)
= \underbrace{d(\mu_1, \mu_1)}_{=0} + d(\mu_1, \mu)\,,
\]
from which the inequality~\eqref{eq:proof_monotony_chernoff_Dexp} follows by
considering $x = \mu_1$.
\end{proof}

\paragraph{A slightly weaker version of Lemma~\ref{lem:duality}, sufficient for our purposes.}
We may now come back  to the proof of the target equality~\eqref{eq:Phi=L} for canonical one-parameter exponential families.
The following slightly weaker version of Lemma~\ref{lem:duality} is enough to yield~\eqref{eq:Phi=L},
given the rewritings~\eqref{eq:cont-d1} and~\eqref{eq:cont-d2}.

\begin{lemma}
\label{lem:duality-dexp}
Consider a canonical one-parameter exponential family $\cD = \cDexp$.
For all $\nu \in \cDexp$,
\[
\forall x \in \cM, \qquad \phi^\star_\nu(x) = d\bigl(x,\Ed(\nu)\bigr)\,.
\]
\end{lemma}

The result of the lemma holds, by conventions, for $x < \mu_-$ or $x > \mu_+$, but does not hold
in general for $x \in \{ \mu_-, \mu_+ \}$. \\

\begin{proof}
By Lemma~\ref{lem:variational}, we only need to show that $\phi^\star_\nu(x) \geq d\bigl(x, \Ed(\nu)\bigr)$.
Given the definition~\eqref{eq:defphistar} of $\phi^\star_\nu$ as a supremum,
it suffices to exhibit a $\lambda^\star \in \R$ such that
\begin{equation} \label{eq:proofPhi=D1}
d\bigl(x, \Ed(\nu)\bigr) = \lambda^\star x - \phi_\nu(\lambda^\star) \,.
\end{equation}
Let $\theta_1 \in \Theta$ be such that $\nu = \nu_{\theta_1}$ and $\theta_2 = (b')^{-1}(x) \in \Theta$ be such that $\Ed(\nu_{\theta_2}) = x$. We will prove~\eqref{eq:proofPhi=D1} with $\lambda^\star = \theta_2 - \theta_1$. Given the closed-form expression of the densities~\eqref{eq:density_cDexp},
the distribution $\nu_{\theta_2}$ is absolutely continuous with respect to $\nu_{\theta_1}$,
with density given by $(\theta_2 - \theta_1) \idR - \bigl(b(\theta_2) - b(\theta_1)\bigr)$.
Therefore, by definition of the Kullback-Leibler divergence,
\begin{align}
\nonumber
d\bigl(x, \Ed(\nu)\bigr) &= \KL( \nu_{\theta_2}, \nu_{\theta_1}) = \E_{\nu_{\theta_2}} \! \left[ \ln \frac{\d \nu_{\theta_2}}{\d\nu_{\theta_1}} \right]
= \E_{\nu_{\theta_2}} \! \Bigl[ (\theta_2 - \theta_1) \, \idR - \bigl(b(\theta_2) - b(\theta_1)\bigr) \Bigr] \\
\label{eq:d:KL:calc}
&= (\theta_2 - \theta_1) \, \Ed(\nu_{\theta_2}) - \bigl(b(\theta_2) - b(\theta_1)\bigr) = \lambda^\star x - \bigl(b(\theta_2) - b(\theta_1)\bigr) \,.
\end{align}
To obtain~\eqref{eq:proofPhi=D1}, it only remains to show that $b(\theta_2) - b(\theta_1) =  \phi_\nu(\lambda^\star)$.
Using the closed-form expressions~\eqref{eq:expression_b(theta)} of $b$ at $\theta_2$ and~\eqref{eq:density_cDexp}
of the density at $\theta_1$, we obtain
\begin{align}
b(\theta_2) &= \ln \int_\R e^{\theta_2 y} \,\d\rho(y) = b(\theta_1) +
\ln \int_\R e^{(\theta_2-\theta_1) y} \overbrace{e^{\theta_1 y - b(\theta_1)} \,\d\rho(y)}^{= \d\nu_{\theta_1}(y) = \d\nu(y)} \nonumber \\
	&= b(\theta_1) + \ln \int_\R e^{\lambda^\star y} \,\d\nu(y) = b(\theta_1) + \phi_\nu(\lambda^\star) \,, 	\label{eq:proofPhi=D4}
\end{align}
which concludes the proof.
\end{proof}

\begin{remark}
A more direct approach bypassing Lemma~\ref{lem:variational} can be followed
with $\cDexp$ models, along the following lines.
The result~\eqref{eq:proofPhi=D4} can be generalized into
\begin{equation}
\label{eq:phi=diff-b}
\forall \theta \in \Theta, \qquad \phi_\nu(\theta-\theta_1) = b(\theta)-b(\theta_1) \,.
\end{equation}
As $b$ is differentiable on $\Theta$, the function $\phi_\nu$ is also differentiable; at $\lambda^\star =
\theta_2-\theta_1$, we have
\[
\phi_\nu'(\lambda^\star) = \phi_\nu'(\theta_2-\theta_1) = b'(\theta_2) = x \,.
\]
Thus, the derivative of the strictly concave function $\Lambda : \lambda \in \R \longmapsto \lambda x - \phi_{\nu}(\lambda)$
vanishes at $\lambda^\star$, which is therefore the argument of its maximum:
$\phi^\star_\nu(x) = \Lambda(\lambda^\star)$.
The closed-form calculation~\eqref{eq:d:KL:calc}
and the rewriting~\eqref{eq:phi=diff-b} then lead to Lemma~\ref{lem:duality-dexp}.
\end{remark}

\subsection{Conditions for general models}	
\label{app:Phi=L_condtion}

In this section, we extend Lemma~\ref{lem:duality}, and thus the target equality~\eqref{eq:Phi=L}, to more general models.
We did so by mimicing the proof of Lemma~\ref{lem:duality}: the result below can certainly be improved.
We extend as follows the definitions of the lower and upper ends $m(\nu)$ and $M(\nu)$ of the closed support $\Supp(\nu)$ of
a distribution $\nu$ over $\R$:
\[
m(\nu) = \inf\bigl(\Supp(\nu)\bigr) \in \R \cup \{-\infty\}
\qquad \mbox{and} \qquad
M(\nu) = \sup\bigl(\Supp(\nu)\bigr) \in \R \cup \{+\infty\}\,.
\]

\begin{lemma}
\label{lem:duality-ter}
Consider a model $\cD$ containing distributions $\nu$ over $\R$ with finite first moments
and with exponential moments: $\e^{\lambda \idR} \in \mathbb{L}^1(\nu)$ for all $\lambda \in \R$.
Assume that the model $\cD$ is stable by exponential reweighting of densities:
for all $\nu \in \cD$, for all $\lambda \in \R$, the distribution $\nu_{\lambda}$ with density
\begin{equation}
\label{eq:ass:nu-lambda}
\frac{\d \nu_\lambda}{\d \nu} = \frac{\e^{\lambda \idR}}{\E_{\nu\!}\big[\e^{\lambda \idR}\big]} \qquad \mbox{with respect to} \ \nu
\end{equation}
also belongs to~$\cD$. Assume also that
$\delta_{x}$, the Dirac mass at~$x$,
belongs to $\cD$ whenever there exists
$\nu \in \cD$ with $x \in \bigl\{ m(\nu), M(\nu)\bigr\} \cap \R$
and $\nu \{ x \} > 0$; put differently, if a distribution $\nu \in \cD$ puts
some probability mass on an end~$x$ of its closed support, then the
Dirac mass at $x$ belongs to $\cD$.

Then, for all $\nu \in \cD$,
\[
\forall x \leq \Ed(\nu), \quad \phi^\star_\nu(x) = \Linfeq(x,\nu)
\qquad \mbox{and} \qquad
\forall x \geq \Ed(\nu), \quad \phi^\star_\nu(x) = \Lsupeq(x,\nu) \,.
\]
\end{lemma}

\begin{proof}
By symmetry and by Lemma~\ref{lem:variational}, we only need to prove that
\begin{equation}	
\label{eq:C3_Phi<=Linf}
\forall x \leq \Ed(\nu), \quad \phi^\star_\nu(x) \geq \Linfeq(x,\nu) \,.
\end{equation}
For $x = \Ed(\nu)$, we have $\phi^\star_\nu\bigl( \Ed(\nu) \bigr) = 0 = \Linfeq\bigl(\Ed(\nu),\nu\bigr)$,
as stated at the end of Appendix~\ref{sec:app-SR-CR} and by taking $\zeta = \nu$ in the infimum defining $\Linfeq$,
respectively.
Before moving to the case $x < \Ed(\nu)$, we establish a few properties of $\phi_\nu$ based on the
assumptions of Lemma~\ref{lem:duality-ter}.
All random variables $\e^{\lambda \idR}$ are $\nu$--integrable, for $\lambda \in \R$, which entails, by
application of a standard theorem of differentiation under the integral sign together with local domination arguments
of the form
\[
\forall \lambda \in (\lambda_-, \lambda_+), \qquad \bigl| \idR \, \e^{\lambda \idR} \bigr| \leq \bigl| \idR \bigr| \, \bigl( \e^{\lambda_- \idR} + \e^{\lambda_+ \idR} \bigr) \leq \bigl( \e^{\idR} + \e^{- \idR} \bigr) \bigl( \e^{\lambda_- \idR} + \e^{\lambda_+ \idR} \bigr)\,,
\]
that $\phi_\nu$ is differentiable over $\R$, with derivative given by
\begin{equation}
\label{eq:cf-phiprimenu}
\phi'_{\nu} : \lambda \in \R \longmapsto \frac{\E_{\nu}\big[\idR \, \e^{\lambda \idR}\big]}{\E_{\nu}\big[\e^{\lambda \idR}\big]} \,.
\end{equation}
Hölder's inequality still entails that $\phi_\nu$ is convex, thus its derivative $\phi_\nu'$ is non-decreasing;
therefore, $\phi_\nu'$ admits a limit $\ell \in \{-\infty\} \cup \R$ at $-\infty$.
Actually, we have $\ell = m(\nu)$, as can be seen by combining the following facts.
First, by definition, $\idR \geq m(\nu)$ $\nu$-a.s., thus $\phi'_{\nu} \geq m(\nu)$, hence $\ell \geq m(\nu)$.
As a consequence, if $\ell = -\infty$, then we also have $m(\nu) = -\infty$.
Otherwise, if $\ell \in \R$, the same arguments as in~\eqref{eq:limit=m(nu)1}--\eqref{eq:limit=m(nu)2} show that
$\idR - \ell \geq 0$ $\nu$-a.s., i.e., $\ell \leq m(\nu)$.

We may now come back to establishing $\phi^\star_\nu(x) \geq \Linfeq(x,\nu)$ in the case $x < \Ed(\nu)$.
We consider three sub-cases,
depending on the respective positions of $x$ and $m(\nu)$. \smallskip

\emph{Case~1: $x > m(\nu)$.}~~The properties of $\phi_\nu$ ensure, exactly as in Case~1 of the proof of Lemma~\ref{lem:duality} (in Appendix~\ref{sec:appcDall}),
the existence of $\lambda^\star$ such that $\phi_\nu'(\lambda^\star) = x$ and
$\phi_\nu^\star(x) = \lambda^\star x - \phi_\nu(\lambda^\star)$.
Given the assumption~\eqref{eq:ass:nu-lambda}, we may consider the distribution $\nu_{\lambda^\star} \in \cD$.
We note, again exactly as in Case~1 of the proof of Lemma~\ref{lem:duality} and given the closed-form expression~\eqref{eq:cf-phiprimenu}
for $\phi_\nu'$, that $\Ed(\nu_{\lambda^\star}) = \phi_\nu'(\lambda^\star)$, thus $\Ed(\nu_{\lambda^\star}) = x$. Finally,
an explicit computation yields
\[
\KL(\nu_{\lambda^\star}, \nu) = \lambda^\star \Ed(\nu_{\lambda^\star})
- \ln \E_{\nu\!}\big[\e^{\lambda \idR}\big] = \lambda^\star x - \phi_\nu(\lambda^\star) = \phi_\nu^\star(x)\,.
\]
By the defining infimum of $\Linfeq(x, \nu)$, we have indeed $\Linfeq(x, \nu) \leq \KL(\nu_{\lambda^\star}, \nu) = \phi_\nu^\star(x)$.
\smallskip

\emph{Case~2: $x = m(\nu)$.}~~In particular, $m(\nu) \in \R$, which allows us to follow the
monotone-convergence arguments of Case~2
of the proof of Lemma~\ref{lem:duality} (in Appendix~\ref{sec:appcDall})
and get the equality $\phi_\nu^\star\bigl(m(\nu)\bigr) = - \ln \nu \bigl\{ m(\nu) \bigr\}$.
Now, for the second part of this sub-case,
we also adapt an argument of the second part of the proof of Lemma~\ref{lm:Linfstr-egal-Linfeq}
(in Appendix~\ref{sec:spec-cDall}), namely, the fact that either there exists at most one distribution $\zeta \in \cD$
absolutely continuous with respect to $\nu$ and satisfying $\Ed(\zeta) \leq m(\nu)$, namely,
$\zeta = \delta_{m(\nu)}$, the Dirac mass at $m(\nu)$. The latter is indeed absolutely continuous with
respect to $\nu$ if and only if $\nu \bigl\{ m(\nu) \bigr\} > 0$.
When $\nu \bigl\{ m(\nu) \bigr\} > 0$, we have $\delta_{m(\nu)} \in \cD$ by the Dirac assumption of the lemma,
so that
\[
\Linfeq\bigl( m(\nu) , \nu \bigr) =
\KL\bigl(\delta_{m(\nu)}, \nu\bigr) = - \ln \nu \bigl\{ m(\nu) \bigr\}\,.
\]
Otherwise, when $\nu \bigl\{ m(\nu) \bigr\} = 0$, the infimum defining $\Linfeq\bigl( m(\nu) , \nu \bigr)$
is either over an empty set or of $+\infty$ values, and thus equals $+\infty = - \ln \nu \bigl\{ m(\nu) \bigr\}$. In both situations,
we obtained $\Linfeq\bigl( m(\nu) , \nu \bigr) = \phi_\nu^\star\bigl(m(\nu)\bigr)$.
\smallskip

\emph{Case~3: $x < m(\nu)$.}~~In particular, $m(\nu) \in \R$ in this sub-case as well,
which allows us to repeat the exact same arguments as in Case~3
of the proof of Lemma~\ref{lem:duality} (in Appendix~\ref{sec:appcDall}):
we may show that both $\Linfeq(x,\nu)$ and $\phi^\star_\nu(x)$ are equal to $+\infty$.
\end{proof}

\section{Proofs for lower bounds (Section~\ref{sec:LB})}
\label{app:LB}

This section provides the detailed proofs that were omitted
when stating our various lower bounds in Section~\ref{sec:LB}.

\subsection{Proof of Lemma~\ref{lem:BI}}
\label{app:lem:BI}

We restate the lemma for the convenience of the reader.
The proof reveals that the inequality actually holds for
limits taken along subsequences $(T_n)_{n \geq 1}$.
Also, we may only relax the assumptions on the bandit models;
e.g., they do not need to be generic and it suffices that they
have different unique optimal arms. (The notion of a generic bandit
problem is defined in the first lines of Section~\ref{sec:LB}.)

\lmfunda*

\begin{proof}
The considered sequence of strategies being consistent on $\cD$, and as $a^\star(\ula) \neq a^\star(\unu)$,
\begin{align*}
& q_T \eqdef \P_{\ula}\bigl( I_T \ne a^\star(\unu) \bigr)
\geq \P_{\ula}\bigl( I_T = a^\star(\ula) \bigr) \lriT 1\,, \\
\mbox{while} \qquad &
p_T \eqdef \P_{\unu}\bigl( I_T \ne a^\star(\unu) \bigr) \lriT 0 \,.
\end{align*}
Note that we introduced above short-hand notation $p_T$ and $q_T$.

The fundamental inequality for lower bounds in bandit problems
(which is a consequence of the chain rule
and of the data-processing inequality for Kullback-Leibler
divergences, see \citealp{GMS19}),
applied for $Z = \ind{I_T \ne a^\star(\unu)}$, exactly states here that
\begin{equation}	\label{eq:chainrule+datacompressing}
\sum_{a=1}^K \E_{\ula}[N_a(T)] \, \KL(\lambda_a, \nu_a) \geq \KL\big(\Ber(q_T),\Ber(p_T)\bigr)\,,
\end{equation}
where we recall that $\Ber(p)$ refers to the Bernoulli distribution with parameter~$p$.
Given the asymptotics of $p_T$ and $q_T$,
\[
\KL\big(\Ber(q_T),\Ber(p_T)\bigr) = q_T \ln \frac{q_T}{p_T} +
(1-q_T) \ln \frac{1-q_T}{1-p_T} \sim - \ln p_T
\qquad \mbox{as} \ T \to +\infty\,.
\]
Put differently,
\[
\frac{1}{T} \ln \P_{\unu}\bigl(I_T \neq a^\star(\unu)\bigr) \sim - \frac{\KL\big(\Ber(q_T),\Ber(p_T)\bigr)}{T}\,.
\]
Combining this limit behavior with the previous inequality leads to the stated result, namely:
\[
\liminf_{T \to +\infty} \frac{1}{T} \ln \P_{\unu}\bigl(I_T \neq a^\star(\unu)\bigr)
\geq - \limsup_{T \to +\infty} \sum_{a=1}^K \frac{\E_{\ula}[N_a(T)]}{T} \, \KL(\lambda_a, \nu_a) \,. \vspace{-.8cm}
\]
\end{proof}

\subsection{Proof of Theorem~\ref{th:LB-ABM10}}
\label{app:th:LB-ABM10}

We restate the theorem for the convenience of the reader
(and recall that the notion of a generic bandit
problem is defined in the first lines of Section~\ref{sec:LB}).

\lbabm*

\begin{proof}
The proof consists of two steps.
The first step is to prove that for a generic bandit
problem $\unu$ in $\cD$ with $K \geq 2$ arms, we have,
\begin{equation}
\label{eq:LB-LinfK}
\liminf_{T \to +\infty} \frac{1}{T} \ln \P_{\unu}\bigl(I_T \ne a^\star(\unu)\bigr)
\geq - \frac{\Linfstr\bigl(\mu_{(K)},\nu^\star\bigr)}{K}\,.
\end{equation}
In the second step, we use this lower bound and the very definition of
the clever exploitation of the pruning of suboptimal arms to get the claimed
bound.

\paragraph{Step~1: lower bound~\eqref{eq:LB-LinfK}.}
We follow a well-established methodology and
consider an alternative bandit problem only differing
from $\unu$ at one arm, namely, at the best arm.
To do so, we set some distribution $\zeta \in \cD$ with $\Ed(\zeta) < \mu_{(K)}$,
if some exists, and define the bandit problem $\ula = (\lambda_1, \ldots, \lambda_K)$
as
\[
\lambda_a = \left\{
\begin{aligned}
\nonumber
& \zeta & \text{if } a = a^\star(\unu), \\
& \nu_a & \text{if } a \ne a^\star(\unu).
\end{aligned}\right.
\]
Observe that $\ula$ is also a generic bandit problem in $\cD$,
that $a^\star(\unu)$ is the worst arm in $\ula$ (and also
that the second best arm of $\unu$ is the optimal arm in $\ula$, but we will not
use this specific fact). Therefore, Lemma \ref{lem:BI} yields, as $\ula$ and $\unu$ only differ at arm $a^\star(\unu)$,
\[
\liminf_{T \to +\infty} \frac{1}{T} \ln \P_{\unu}\bigl(I_T \neq a^\star(\unu)\bigr) \geq - \limsup_{T \to +\infty} \frac{\E_{\ula}[N_{a^\star(\unu)}(T)]}{T}
\, \KL(\lambda_{a^\star(\unu)}, \nu^\star) \,,
\]
where we recall that $\nu^\star = \nu_{a^\star(\unu)}$.
Given that $a^\star(\unu)$ is the worst arm of $\ula$,
and since by assumption, the sequence of strategies is balanced against the worst arm,
\[
\limsup_{T \to +\infty} \frac{1}{T} \, \E_{\ula} \bigl[ N_{a^\star(\unu)}(T) \bigr]
\leq \frac{1}{K}\,,
\]
proving that
\[
\liminf_{T \to +\infty} \frac{1}{T} \ln \P_{\unu}\bigl(I_T \neq a^\star(\unu)\bigr) \geq - \frac{\KL(\zeta,\nu^\star)}{K}\,.
\]
The claimed inequality~\eqref{eq:LB-LinfK} follows from
taking the supremum in the right-hand side over
distributions $\zeta \in \cD$ with $\Ed(\zeta) < \mu_{(K)}$.

\paragraph{Step~2: clever exploitation of pruning.}
For each $k \in \{2,\ldots,K-1\}$,
define $\unu'_{1:k}$ as the subproblem of $\unu$ obtained
by keeping the $k$ best arms and dropping the $K-k$ worst arms.
Use the definition of clever exploitation of pruning of suboptimal arms
and apply~\eqref{eq:LB-LinfK} to $\unu'_{1:k}$ to get
\[
\liminf_{T \to +\infty} \frac{1}{T} \ln \P_{\unu}\bigl(I_T \ne a^\star(\unu)\bigr)
\geq
\liminf_{T \to +\infty} \frac{1}{T} \ln \P_{\unu'_{1:k}}\Bigl(I_T \ne a^\star\bigl(\unu'_{1:k}\bigr)\Bigr)
\geq
- \frac{\Linfstr\bigl(\mu_{(k)},\nu^\star\bigr)}{k}\,.
\]
Taking the maximum of all lower bounds exhibited as $k$ varies between $2$ and $K$,
we proved the claimed result.
\end{proof}

\subsection{Proof of the normality of the models $\cDall$ and $\cDexp$}
\label{sec:proof-normality}

In this section, we show that $\cDall$ and canonical one-parameter exponential models are normal.
For the convenience of the reader, we first restate the definition of normality.

\definormality*

\begin{proposition}
$\cDall$ is a normal model.
\end{proposition}

\begin{proof}
We fix $\nu \in \cDall$, a real $x \geq \Ed(\nu)$, and $\varepsilon > 0$.
Recall the piece of notation $M(\nu)$ for the upper end of the support of~$\nu$, as introduced in Appendix~\ref{sec:spec-cDall}.
As in Case~3 of the proof of Lemma~\ref{lem:duality} (in Appendix~\ref{sec:appcDall}),
we note that when $x \geq M(\nu)$, there exists no distribution $\zeta \in \cDall$ absolutely continuous with respect to $\nu$ and
such that $\Ed(\zeta) > x$; hence, both infima in Definition~\ref{defi:normality} equal $+\infty$.
We now tackle the case where $\Ed(\nu) \leq x < M(\nu)$.
For all $\delta > 0$, we introduce
\[
x'_\delta = \min\biggl\{ x + \delta, \, \frac{x + M(\nu)}{2} \biggr\} < M(\nu) \,.
\]
Case~1 of the proof of Lemma~\ref{lem:duality} and Lemma~\ref{lem:variational}
reveal (by symmetry) that for each $\delta > 0$, there exists a distribution $\zeta_\delta \in \cDall$
with expectation $x'_\delta$ and such that $\Lsupeq(x'_\delta,\nu) = \phi^\star_\nu(x'_\delta) = \KL(\zeta_\delta,\nu)$.
By Lemma~\ref{lm:Linfstr-egal-Linfeq}, $\Lsupeq(x'_\delta,\nu) = \Lsupstr(x'_\delta,\nu)$ and
$\Lsupstr(\,\cdot\,,\nu)$ is continuous on $\bigl( - \infty, M(\nu) \bigr)$.
Putting all these elements together, we obtain
\begin{align*}
\Lsupstr(x,\nu) = \lim_{\delta \to 0} \Lsupstr(x'_\delta,\nu) = & \liminf_{\delta \to 0} \KL(\zeta_\delta,\nu) \\
& \geq \inf \bigl\{ \KL(\zeta_\delta,\nu) : \ \delta \in (0,\varepsilon)  \bigr\} \\
& \geq \inf \bigl\{ \KL(\zeta,\nu) : \ \zeta \in \cD \ \ \mbox{\rm s.t.} \ \ x + \varepsilon > \Ed(\zeta) > x  \bigr\}\,,
\end{align*}
where the first inequality is by the very definition of a $\liminf$.
\end{proof}

\begin{proposition}
All canonical one-parameter exponential models $\cDexp$ are normal.
\end{proposition}

\begin{proof}
The proof consists of rewriting $\Lsupstr$ as $d$, as indicated by~\eqref{eq:cont-d2}, and using the
regularity properties for~$d$ exhibited in Appendix~\ref{app:cDexp}.
We fix $\nu \in \cDexp$, a real $x \geq \Ed(\nu)$, and $\varepsilon > 0$.
When $x \geq M(\nu)$, the same argument as in the previous proposition shows that
both infima equal $+\infty$. For $x < M(\nu)$, we
introduce $\delta \in (0, \mu_+ - x)$
and write
\begin{align*}
\Lsupstr(x,\nu) = d\bigl(x, \Ed(\nu)\bigr)
& = \lim_{\delta \to 0} d\bigl(x+\delta, \Ed(\nu)\bigr) \\
& = \inf \Bigl\{ d\bigl(x+\delta, \Ed(\nu)\bigr) : \delta \in (0,\varepsilon) \Bigr\} \\
& = \inf \bigl\{ \KL(\zeta,\nu) : \ \zeta \in \cD \ \ \mbox{\rm s.t.} \ \ x + \varepsilon > \Ed(\zeta) > x  \bigr\}\,,
\end{align*}
where the second and third equalities follow, respectively, by continuity of $d\bigl(\,\cdot\,, \Ed(\nu)\bigr)$ on $\cM$
and by the fact that this function is non-decreasing on $(x, \mu_+) \subset [\Ed(\nu), \mu_+)$,
and the final equality is by the rewriting~\eqref{eq:d-KL}.
\end{proof}

\subsection{Proof of Theorem~\ref{th:LBmonotone}}
\label{app:th:LBmonotone}

We restate the theorem for the convenience of the reader
(and recall that the notion of a generic bandit
problem is defined in the first lines of Section~\ref{sec:LB}).

\thlbmonotone*

\begin{proof}
We fix a generic bandit $\unu$ in $\cD$ and consider the following sets of alternative bandit problems, indexed by triplets
$(k,j,x)$ satisfying $2 \leq k \leq K$ and $2 \leq j \leq k$, as well as $x \in [\mu_{(j)}, \mu_{(j-1)})$:
\[
\Alt_{k, j, x}(\unu) = \Bigl\{ \ula \text{ in } \cD : \ \Ed\bigl(\lambda_{(1)}\bigr) < x < \Ed\bigl(\lambda_{(k)}\bigr) < \mu_{(j-1)}
\ \text{ and } \ \lambda_a = \nu_a \,\, \text{ for } \,\, a \notin \big\{(1), (k) \big\} \Bigr\} \,;
\]
in particular, an alternative problem $\ula$ in $\Alt_{k, j, x}(\unu)$ only differ from the original bandit problem $\unu$ at the best arm $(1)$ and
at the $k$--th best arm $(k)$. Given $x \in \bigl[\mu_{(j)}, \mu_{(j-1)} \bigr)$ and $\Ed\bigl(\lambda_{(1)}\bigr) < x$, arm $(1)$ is at best the $j$--th best arm of $\ula$, but it can be possibly worse.
Similarly, the same condition on $x$ and the fact that $x < \Ed\bigl(\lambda_{(k)}\bigr)$ implies that arm $(k)$ is exactly the $j-1$--th
best arm of $\ula$. Both facts are illustrated on Figure~\ref{fig:Altkjx}.
\begin{figure}[t]
%
%
\centering \includegraphics[scale=0.75]{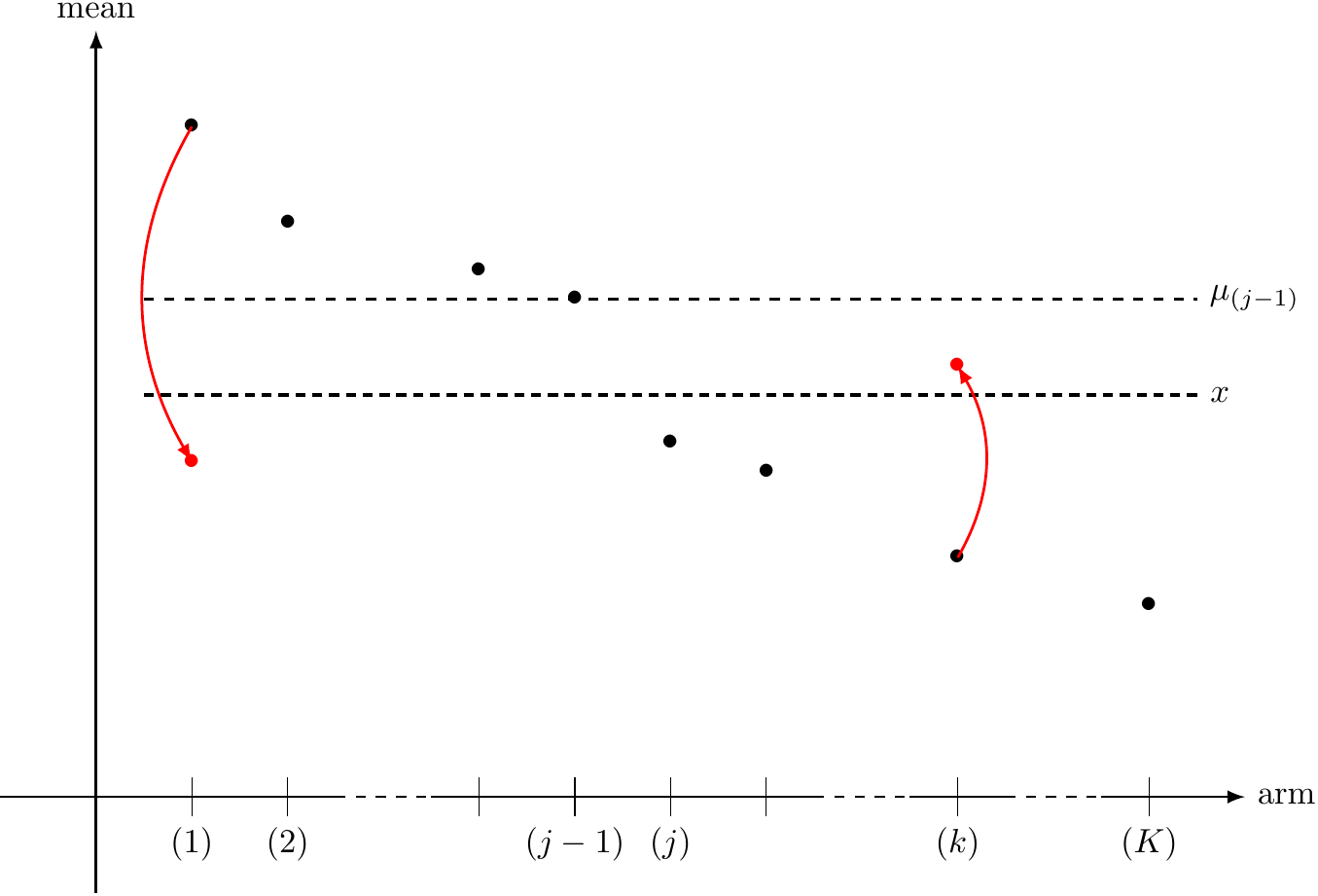}
\caption{Original bandit problem $\unu$ (in dark) and modifications made to arms $(1)$ and $(k)$ to obtain an alternative bandit problem
$\textcolor{red}{\ula} \in \Alt_{k, j, x}(\unu)$ (in {\textcolor{red}{red}}): in $\textcolor{red}{\ula}$, arm $(k)$ is the
$j-1$--th best arm, while arm $(1) = a^\star(\unu)$
is at best the $j$--th best arm. }
\label{fig:Altkjx}
\end{figure}
Thus, by monotonicity of the strategy,
\[
\limsup_{T \to +\infty} \frac{\E_{\ula}\bigl[N_{(k)}(T)\bigr]}{T} \leq \frac{1}{j-1}
\qquad \text{and} \qquad
\limsup_{T \to +\infty} \frac{\E_{\ula}\bigl[N_{(1)}(T)\bigr]}{T} \leq \frac{1}{j} \,.
\]
Given that the optimal arm in $\ula$ is different from the optimal arm $(1)$ of $\unu$,
Lemma~\ref{lem:BI} may be applied; together with the two upper bounds above, it yields
\[
\liminf_{T \to +\infty} \frac{1}{T} \ln \P_{\unu}\bigl(I_T \ne a^\star(\unu) \bigr) \geq - \biggl( \frac{\KL\bigl(\lambda_{(k)}, \nu_{(k)}\bigr)}{j-1} + \frac{\KL\bigl(\lambda_{(1)}, \nu^\star\bigr)}{j} \biggr) \,.
\]
We can now take the infimum over all bandit problems $\ula \in \Alt_{k, j, x}(\unu)$
and obtain the following lower bound, where we define a quantity $\cI_{k,j,x}(\unu)$:
\[
\liminf_{T \to +\infty} \frac{1}{T} \ln \P_{\unu}\bigl(I_T \ne a^\star(\unu) \bigr)
\geq - \inf_{\ula \in \Alt_{k, j, x}(\unu)} \bigg\{ \frac{\KL\bigl(\lambda_{(k)}, \nu_{(k)}\bigr)}{j-1} + \frac{\KL\bigl(\lambda_{(1)}, \nu^\star\bigr)}{j} \bigg\}
\eqdef - \cI_{k,j,x}(\unu) \,.
\]
We prove below that
\begin{equation}
\label{eq:Ijkx}
\cI_{k,j,x}(\unu) = \frac{\Lsupstr\bigl(x, \nu_{(k)}\bigr)}{j-1} + \frac{\Linfstr\bigl(x, \nu^\star\bigr)}{j}\,,
\end{equation}
from which the lower bound claimed in Theorem~\ref{th:LBmonotone} will follow, by taking the supremum
of $- \cI_{k,j,x}(\unu)$ first over $x \in \bigl[\mu_{(j)}, \mu_{(j-1)}\bigr)$, then the maximum over $2 \leq j \leq k$,
and finally, the maximum over $2 \leq k \leq K$.

We now prove~\eqref{eq:Ijkx}.
The infimum over $\ula \in \Alt_{k, j, x}(\unu)$ may be split into two separate infima, respectively
over $\lambda_{(k)}$ and $\lambda_{(1)}$; given that each term of the sum of $\KL$ only depends either on
$\lambda_{(k)}$, or on $\lambda_{(1)}$, but not on both, we may write
\begin{align*}
\cI_{k,j,x}(\unu) &= \inf_{\substack{ \lambda_{(1)}, \lambda_{(k)} \in \cD \,: \\ \Ed(\lambda_{(1)}) < x \\ x < \Ed(\lambda_{(k)}) < \mu_{(j-1)}}} \bigg\{ \frac{\KL\bigl(\lambda_{(k)}, \nu_{(k)}\bigr)}{j-1} + \frac{\KL\bigl(\lambda_{(1)}, \nu^\star\bigr)}{j} \bigg\} \\
	&= \frac{1}{j-1} \underbrace{\inf_{\substack{ \lambda_{(k)} \in \cD \,: \\ x < \Ed(\lambda_{(k)}) < \mu_{(j-1)}}} \KL\bigl(\lambda_{(k)}, \nu_{(k)}\bigr)}_{= \Lsupstr(x, \nu_{(k)})}
+ \frac{1}{j} \underbrace{\inf_{\substack{ \lambda_{(1)} \in \cD \,: \\ \Ed(\lambda_{(1)}) < x}}  \KL\bigl(\lambda_{(1)}, \nu^\star\bigr)}_{= \Linfstr(x, \nu^\star)}\,,
\end{align*}
where we obtained $\Linfstr\bigl(x, \nu^\star\bigr)$ by definition while we relied on the
normality of the model (Definition~\ref{defi:normality}) to obtain $\Lsupstr\bigl(x, \nu_{(k)}\bigr)$.
We did so with $\varepsilon = \mu_{(j-1)} - x$, which is indeed
positive as we considered $x < \mu_{(j-1)}$.
\end{proof}

\subsection{Proof of Theorem~\ref{th:LB-conj-GK16}}
\label{app:thLBconjGK}

We restate the theorem for the convenience of the reader
(and recall that the notion of a generic bandit
problem is defined in the first lines of Section~\ref{sec:LB}).

\thLBconjGK*

\begin{proof} Let $\unu$ be a generic bandit problem.
We fix $k \neq a^\star(\unu)$ and $x \in [\mu_{k}, \mu^\star]$,
and prove that
\[
\liminf_{T \to +\infty} \frac{1}{T} \ln \P_{\unu} \bigl(I_T \neq a^\star(\unu)\bigr) \geq
- \max \bigl\{ \Lsupstr(x, \nu_k),  \Linfstr(x, \nu^\star) \bigr\}\,,
\]
from which the stated lower bound follows, by taking suprema.
To do so, we consider the set of alternative bandit problems
\[
\Alt_{k, x}(\unu) = \Bigl\{ \ula \text{ in } \cD : \ \Ed\bigl(\lambda_{a^\star(\unu)}\bigr) < x <\Ed(\lambda_k)
\ \text{ and } \ \lambda_a = \nu_a \,\, \text{ for } \,\, a \notin \bigl\{ a^\star(\unu), k\bigr\} \Bigr\}\,;
\]
it is composed of bandit problems, only differing from $\unu$ at arms $a^\star(\unu)$ and $k$,
and for which arm $k$ is better than arm $a^\star(\unu)$, with associated expectations separated by $x$.
In particular, the optimal arm in $\ula$ is different from the optimal arm $a^\star(\unu)$ of $\unu$.
Lemma~\ref{lem:BI} may therefore be applied; it states that
\begin{align*}
\lefteqn{\liminf_{T \to +\infty} \frac{1}{T} \ln \P_{\unu} \bigl( I_T \neq a^\star(\unu) \bigr)} \\
& \geq - \limsup_{T \to +\infty}
\frac{\E_{\ula}\bigl[N_k(T)\bigr]}{T} \, \KL(\lambda_k, \nu_k) +
\frac{\E_{\ula}\bigl[N_{a^\star(\unu)}(T)\bigr]}{T} \, \KL\bigl(\lambda_{a^\star(\unu)}, \nu_{a^\star(\unu)}\bigr) \\
& \geq - \max \Big\{ \KL(\lambda_k, \nu_k) , \, \KL\bigl(\lambda_{a^\star(\unu)}, \nu_{a^\star(\unu)} \bigr) \Big\}\,,
\end{align*}
where we used, for the second inequality, the crude upper bound $N_k(T) + N_{a^\star(\unu)}(T) \leq T$.
Taking the supremum of the obtained lower bound over all $\ula \in \Alt_{k, x}(\unu)$ leads to
the following inequality, where we define the short-hand notation $\cI_{k,x}(\unu)$:
\[
\liminf_{T \to +\infty} \frac{1}{T} \ln \P_{\unu}(I_T \neq a^\star(\unu)) \geq
- \inf_{\ula \in \Alt_{k, x}(\unu)} \max \Bigl\{ \KL(\lambda_k, \nu_k) , \KL\bigl(\lambda_{a^\star(\unu)}, \nu_{a^\star(\unu)}\bigr) \Bigr\}
\eqdef - \cI_{k,x}(\unu) \,.
\]
The proof is concluded below by showing that $\cI_{k,x}(\unu) = \max \bigl\{ \Lsupstr(x, \nu_k),  \Linfstr(x, \nu^\star) \bigr\}$.

As in the proof of Theorem~\ref{th:LBmonotone} (see Appendix~\ref{app:th:LBmonotone}),
we use a separation of the infima, in the abstract form,
for two functions~$f$ and~$g$,
\[
\inf_{u,v} \,\, \max\bigl\{ f(u),\,g(v) \bigr\} = \max \Bigl\{ \inf_u f(u)\,,\, \inf_v g(v) \Bigr\}\,.
\]
Here, by definition of $\Alt_{k, x}(\unu)$,
\begin{align*}
\cI_{k,x}(\unu) &= \inf_{\substack{\lambda_{a^\star(\unu)}, \lambda_k \in \cD \\ \Ed(\lambda_{a^\star(\unu)}) < x \\ \Ed(\lambda_k) > x}}
\,\, \max \Big\{ \KL(\lambda_k, \nu_k) , \,\, \KL\bigl(\lambda_{a^\star(\unu)}, \nu_{a^\star(\unu)}\bigr) \Big\}  \\
& = \max \left\{
\inf_{\substack{\lambda_k \in \cD \\ \Ed(\lambda_k) > x}} \KL(\lambda_k, \nu_k), \,\,
\inf_{\substack{\lambda_{a^\star(\unu)} \in \cD \\ \Ed(\lambda_{a^\star(\unu)}) < x}} \KL\bigl(\lambda_{a^\star(\unu)}, \nu_{a^\star(\unu)}\bigr)
\right\} \\
& = \max \Bigl\{ \Lsupstr(x, \nu_k),  \Linfstr(x, \nu^\star) \Bigr\} \,,
\end{align*}
which concludes the proof.
\end{proof}

\section{Additional comments for the literature review}
\label{app:literature}

This appendix is devoted to additional discussions concerning the fixed-budget literature. More precisely, we discuss in detail two gap-based lower bounds that we believe are somewhat detached from the spirit of the article, namely, the minimax lower bound of \cite{CL16} in Appendix~\ref{app:CL16} and the Bretagnolle-Huber technique in Appendix~\ref{app:BH}.

\subsection{The minimax lower bound of~\cite{CL16}}
\label{app:CL16}

\citet[Theorem~1]{CL16} proved (slightly stronger versions of)
the following (non-asymptotic) minimax lower bound.
Consider the model $\cDber$ of Bernoulli distributions $\Ber(p)$ with parameters $p \in [1/4, \, 3/4]$.
For all sequences of strategies that are consistent on $\cDber$, for all $T \geq 0.14 \, K^4 \ln(6KT)$,
\begin{equation}	
\label{eq:LB-CL16}
\exists \, \unu \text{ \rm in } \cDber, \qquad \frac{1}{T} \ln \P_{\unu} \bigl( I_T \ne a^\star(\unu) \bigr)
\geq - \frac{400}{\ln K} \Biggl(\sum_{a \neq a^\star(\unu)} \frac{1}{\Delta_a^2} \Biggr)^{-1} - \frac{\ln 6}{T} \,,
\end{equation}
where, of course, we may rather use the weaker lower bound based on
\[
- \Biggl(\sum_{a \neq a^\star(\unu)} \frac{1}{\Delta_a^2} \Biggr)^{-1} \geq
- \min_{2 \leq k \leq K} \frac{\Delta_{(k)}^2}{k}\,.
\]
However, the bound~\eqref{eq:LB-CL16} is different in nature from the lower bounds considered in this article, as first and foremost,
it only guarantees a $1/\ln K$ improvement of the lower bound~\eqref{eq:LB-ABM10} of \cite{ABM10} for a single bandit problem $\unu$
(actually belonging to a known collection of $K$ bandit problems).
This is in strong contrast with the uniform instance-dependent lower bounds presented in this article:
bounds holding simultaneously for all bandit problems of a given model.
Second, the proof of the result (see the simpler proof provided below for Proposition~\ref{prop:CL16} stated next)
is truly gap-based and does not seem to extend in any obvious way to non-parametric models.

As mentioned above, the proof of~\eqref{eq:LB-CL16} in \cite{CL16} uses only
$K$ different bandit problems in $\cDber$.
We may therefore resort to the pigeonhole principle to
exchange, in some sense, the ``for all $T \geq 0.14 K^4 \ln(6KT)$'' and ``there exists $\unu$ in $\cDber$''
parts. More precisely, we obtain, from~\eqref{eq:LB-CL16}
the following proposition. For the sake of completeness, we provide a self-contained proof of this proposition
closely following the original arguments by \cite{CL16},
except for the change-of-measure argument, for which we rather resort to Lemma~\ref{lem:BI}.
Doing so, we are able to improve the numerical factor $400$ that would follow from~\eqref{eq:LB-CL16}
into a smaller factor of $30$.

\begin{proposition}
\label{prop:CL16}
Fix $K \geq 3$ and
consider the model $\cDber$ of Bernoulli distributions $\Ber(p)$ with parameters $p \in [1/4, \, 3/4]$.
For all consistent sequences of strategies on $\cDber$,
there exists an increasing sequence of budgets $(T_n)_{n \geq 1}$ such that
\begin{equation}
\label{eq:prop:CL16}
\exists \, \unu \text{ {\rm in} } \cDber, \qquad
\liminf_{n \to +\infty} \frac{1}{T_n} \ln \P_{\unu} \bigl( I_{T_n} \ne a^\star(\unu) \bigr)
\geq - \frac{30}{\ln K} \Biggl(\sum_{a \neq a^\star(\unu)} \frac{1}{\Delta_a^2} \Biggr)^{-1} \,.
\end{equation}
\end{proposition}

\begin{proof}
We consider some base Bernoulli bandit problem $\unu^{\base} = \bigl( \nu_1^{\base}, \ldots, \nu_K^{\base} \bigr)$, where
\[
\nu^{\base}_1 = \Ber(1/2) \qquad \mbox{and} \qquad \forall j \in \{2, \ldots, K\}, \quad \nu_j^{\base} = \Ber(p_j)\,,
\]
for parameters $p_j \in [1/4, \, 1/2)$ to be specified later.
For each $k \in \{2, \ldots, K\}$, we then define the alternative bandit problem $\unu^{(k)}
= \bigl( \nu^{(k)}_1, \ldots, \nu^{(k)}_K \bigr)$ as follows:
\[
\nu_j^{(k)} = \left\{
\begin{aligned}
\nonumber
& \Ber(1-p_k) & \text{if } j=k, \\
& \nu_j^{\base} & \text{if } j \ne k.
\end{aligned}\right.
\]
Given the constraints on the $p_j$,
the unique optimal arm of $\unu^{\base}$ is $a^\star\bigl( \unu^{\base} \bigr) = 1$,
while the unique optimal arm of $\unu^{(k)}$ is $a^\star\bigl( \unu^{(k)} \bigr) = k$.
We introduce, for a given bandit problem $\unu$
\[
H(\unu) \eqdef \sum_{a \neq a^\star(\unu)} \frac{1}{\Delta_a^2}\,;
\]
the right-hand side of~\eqref{eq:prop:CL16} may be rewritten as $(34/\ln K) \, H(\unu)^{-1}$.
The suboptimality gaps of the arms of $\unu^{\base}$ equal
$\Delta^{\base}_{j} = 1/2 - p_j$ for $j \ne 1$,
while the ones of $\unu^{(k)}$ equal
\begin{align}
\nonumber
\forall j \ne k, \qquad \Delta^{(k)}_j & = 1-p_k-p_j = (1/2-p_k) + (1/2-p_j) = \Delta^{\base}_k + \Delta^{\base}_j\,, \\
\label{eq:Hnuk}
\mbox{thus} \qquad
H\bigl( \unu^{(k)} \bigr) & = \sum_{j \ne k} \frac{1}{ \bigl( \Delta^{\base}_k + \Delta^{\base}_j \bigr)^2}\,.
\end{align}

The proof is decomposed in two steps. First, we show that for all values of the $p_j$ abiding by the constraints
and for all weights $u_2,\ldots,u_K$ such that $u_j \geq 0$ for all $j$ and $u_1 + \ldots + u_K = 1$,
there exists $k^\star \in \{2, \ldots, K\}$ such that there exists an increasing sequence of budgets $(T_n)_{n \geq 1}$ with
\begin{equation}
\label{proof:CL16-part1}
\liminf_{n \to +\infty} \frac{1}{T_n} \ln \P_{\unu^{(k^\star)}} \bigl( I_{T_n} \ne k^\star \bigr)
\geq - 9 \, u_{k^\star} \bigl( \Delta^{\base}_{k^\star} \bigr)^2  \,.
\end{equation}
Then, we set specific values of the $u_j$ and $p_j$ to get
\begin{equation}
\label{proof:CL16-part2}
\forall k \in \{2, \ldots, K\}, \qquad
u_k \bigl( \Delta^{\base}_{k} \bigr)^2 \leq \frac{10}{3 \ln K} \, H\bigl( \unu^{(k)} \bigr)^{-1}\,.
\end{equation}
Proposition~\ref{prop:CL16} follows by combining~\eqref{proof:CL16-part1} and~\eqref{proof:CL16-part2}.
\smallskip

\noindent\emph{Part 1: Proof of~\eqref{proof:CL16-part1}.}~~For all $T \geq 1$,
\[
\sum_{k = 2}^K \frac{\E_{\unu^{\base}}\bigl[N_k(T)\bigr]}{T} \leq 1 = \sum_{k=2}^K u_k\,;
\]
therefore, for all $T \geq 1$, there exists $k_T \in \{2, \ldots, K\}$ such that
$\E_{\unu^{\base}}\bigl[N_{k_T}(T)\bigr]/T \leq u_{k_T}$.
By the pigeonhole principle, there exists $k^\star \in \{2, \ldots, K\}$
and an (infinite) increasing sequence $(T_n)_{n \geq 1}$ of integers such that
$k_{T_n} = k^\star$ for all $n \geq 1$. In particular,
\[
\limsup_{n \to +\infty} \frac{\E_{\unu^{\base}}\bigl[N_{k^\star}(T_n)\bigr]}{T_n} \leq u_{k^\star}\,.
\]
Since $\unu^{\base}$ and $\unu^{(k^\star)}$ only differ at arm~$k^\star$,
an application of Lemma~\ref{lem:BI} along subsequences (see the initial comments in Appendix~\ref{app:lem:BI})
guarantees that
\begin{align*}
\liminf_{n \to +\infty} \frac{1}{T_n} \ln \P_{\unu^{(k^\star)}} \bigl( I_{T_n} \ne k^\star \bigr) &
\geq - \left( \limsup_{n \to +\infty} \frac{\E_{\unu^{\base}}\bigl[N_{k^\star}(T_n)\bigr]}{T_n} \right)
\KL\bigl(\Ber(1-p_{k^\star}), \Ber(p_{k^\star}) \bigr) \\
& \geq - u_{k^\star} \times 9 \, (1/2-p_{k^\star})^2 = - 9 \, u_{k^\star} \, \bigl( \Delta^{\base}_{k^\star} \bigr)^2 \,,
\end{align*}
where, in the last inequality, we used that for all $x \in [1/4, \, 1/2)$,
\[
\KL\bigl(\Ber(1-x), \Ber(x)\bigr) = (1-x) \ln \frac{1-x}{x} + x \ln \frac{x}{1-x} \leq 9 \, \biggl(\frac{1}{2}-x\biggr)^{\!\! 2} \,.
\]

\noindent\emph{Part 2: Proof of~\eqref{proof:CL16-part2}.}~~We set, for $j \in \{2, \ldots, K\}$,
\[
u_j = \frac{U}{\bigl( \Delta^{\base}_{j} \bigr)^2 \, H\bigl( \unu^{(j)} \bigr)}\,, \qquad \mbox{where} \qquad
U = \left( \sum_{k=2}^K \frac{1}{\bigl( \Delta^{\base}_{k} \bigr)^2 \, H\bigl( \unu^{(k)} \bigr)} \right)^{\!\! -1}.
\]
Then, $u_k \bigl( \Delta^{\base}_{k} \bigr)^2 = H\bigl( \unu^{(k)} \bigr)^{-1} U$ for all $k \in \{2, \ldots, K\}$.
To get the desired result, it suffices to guarantee that $U \leq 10/(3 \ln K)$.
To do so, we consider the same values as in \cite{CL16} for the $p_j$, i.e., we set, for $j \in \{2, \ldots, K\}$,
\[
p_j = \frac{1}{2} - \frac{j}{4K} \qquad \mbox{or, equivalently,} \qquad \Delta^{\base}_j = \frac{j}{4K} \,.
\]
We show first that $\bigl( \Delta^{\base}_{k} \bigr)^2 \, H\bigl(\nu^{(k)}\bigr) \leq 2k$, for all $k \in \{2, \ldots, K\}$.
Indeed, by~\eqref{eq:Hnuk} and by lower bounding $\Delta^{\base}_k + \Delta^{\base}_j$ either by $\Delta^{\base}_k$
or $\Delta^{\base}_j$, we get
\begin{align*}
& \bigl( \Delta^{\base}_{k} \bigr)^2 \,
H\bigl( \unu^{(k)} \bigr) =
\sum_{j < k} \frac{\bigl( \Delta^{\base}_{k} \bigr)^2}{ \bigl( \Delta^{\base}_k + \Delta^{\base}_j \bigr)^2}
+ \sum_{j > k} \frac{\bigl( \Delta^{\base}_{k} \bigr)^2}{ \bigl( \Delta^{\base}_k + \Delta^{\base}_j \bigr)^2} \\
\leq \ & k-1 + \sum_{j > k} \frac{\bigl( \Delta^{\base}_{k} \bigr)^2}{ \bigl( \Delta^{\base}_j \bigr)^2}
= k-1 + \sum_{j > k} \frac{k^2}{j^2} \leq k-1 + k^2 \int_k^K\frac{1}{v^2} \d v \leq 2k\,.
\end{align*}
Finally,
\[
U \leq \left( \sum_{k=2}^K \frac{1}{2k} \right)^{\!\! -1}
\leq \left( \int_2^{K+1} \frac{1}{2v} \d v \right)^{\!\! -1}
= 2 \, \bigl( \ln(K+1) - \ln 2 \bigr)^{-1} \leq \frac{10}{3 \ln K}\,,
\]
where the final inequality holds since $K \geq 3$.
\end{proof}

\subsection{The Bretagnolle-Huber technique by \citet[Section~5.2]{KCG16}}
\label{app:BH}

\citet[Section~5.2]{KCG16} provide an interesting series of results relying
on the so-called Bretagnolle-Huber inequality recalled below in~\eqref{eq:BH};
we state one of their lower bounds in Corollary~\ref{cor:16}.
But as we argue in this section, the methodology followed seems extremely specific to the case
of parametric models where Kullback-Leibler divergences
could be controlled (lower bounded and upper bounded) in terms of gaps,
like the model $\cD_{\sigma^2}$ of Gaussian distributions with
a fixed variance $\sigma^2 > 0$.
In particular, we state in Proposition~\ref{prop:LB-BH} what would be the
straightforward extension to non-parametric models of the Gaussian results of \cite[Section~5.2]{KCG16},
and we immediately discuss after this statement why this extension lacks interpretability and interest.
Proposition~\ref{prop:LB-BH} considers any sequence of strategies (not necessarily consistent)
and provides an asymptotic bound; however, it does not directly control
the target probability of error $\P_{\unu}\bigl(I_T \neq a^\star(\unu)\bigr)$, but a larger quantity.
A proof of Proposition~\ref{prop:LB-BH} is provided at the end of this section.

\begin{proposition}
\label{prop:LB-BH}
Fix $K \geq 2$, a model $\cD$, and any sequence of strategies. Let $\unu$ be a bandit problem in $\cD$ with a unique optimal arm. Consider, for each $k \neq a^\star(\unu)$, a distribution $\zeta_k \in \cD$ such that $\Ed(\zeta_k) > \mu^\star$.
For $k \neq a^\star(\unu)$, denote by $\unu^{(k)}$ the bandit problem obtained from $\unu$ by changing the distribution of arm
$k$ into $\zeta_k$.
For all $T \geq 1$,
\[
\frac{1}{T} \ln \max \biggl\{ \P_{\unu}\bigl(I_T \neq a^\star(\unu)\bigr),
\, \max_{k \ne a^\star(\unu)} \P_{\unu^{(k)}}\bigl(I_T \neq k\bigr) \biggr\}
\geq  - \Biggl( \sum_{a \neq a^\star(\unu)} \frac{1}{\KL(\nu_a, \zeta_a)}  \Biggr)^{-1} - \frac{\ln 4}{T}\,.
\]
\end{proposition}

\paragraph*{Lack of interpretability of the bound for general models.}
To derive an interesting and interpretable bound from this result, one needs to choose carefully the distributions $\zeta_k$. There is a tradeoff between obtaining a large lower bound by choosing $\zeta_k$ as close as possible to $\nu_k$ in terms of Kullback-Leibler divergences, and controlling the maximum of the misidentification probabilities: when $\zeta_k$ gets closer to $\nu_k$ while abiding
by the constraint $\Ed(\zeta_k) > \mu^\star$, the probability $\P_{\unu^{(k)}}\bigl(I_T \neq k\bigr)$
becomes larger, and should even intuitively converge to $1/2$. In any case,
the target error $\P_{\unu}\bigl(I_T \neq a^\star(\unu) \bigr)$ should get dominated
by $\P_{\unu^{(k)}}\bigl(I_T \neq k\bigr)$ and the obtained bound is likely to be uninformative
on the target error, due to the maximum in the left-hand side.
This tradeoff seems to be unsolvable in general, unless there exist some specific properties for the Kullback-Leibler
divergence of the model, as we illustrate below for a Gaussian model, which was the
setting considered by~\citet[Section~5.2]{KCG16}.

Another intuitive issue with the bound of Proposition~\ref{prop:LB-BH} is that it involves Kullback-Leibler divergences
with arguments in reverse order compared to the lower bounds presented in Section~\ref{sec:LB}. Indeed, taking the supremum of
the lower bound over distributions $\zeta_k$ such that $\Ed(\zeta_k) > \mu^\star$ would lead to a complexity in terms of
the $\Ksupstr\bigl(\nu_k, \mu^\star\bigr)$, where
\[
\Ksupstr(\nu, x) \eqdef \inf \bigl\{ \KL(\nu, \zeta) : \zeta \in \cD \ \ \mbox{s.t.}
\ \ \Ed(\zeta) > x \bigr\}\,,
\]
rather than in terms of the $\Lsupstr\bigl(\mu^\star, \nu_k\bigr)$.
Our intuition, given all bounds presented in this article,
is that the $\Ksupstr\bigl(\nu_k, \mu^\star\bigr)$ would not form the correct notion
of complexity for the fixed-budget best-arm identification.

\paragraph*{How \citet[Section~5.2]{KCG16} could exploit Proposition~\ref{prop:LB-BH} in the Gaussian case.}
Yet, in the case of the model $\cD_{\sigma^2}$ of Gaussian distributions with
a fixed variance $\sigma^2 > 0$, for which $\KL$ is symmetric,
Proposition~\ref{prop:LB-BH} admits an interesting corollary, corresponding\footnote{The maximum of
the left-hand side of Corollary~\ref{cor:16} is present, but somewhat discrete, in the
Theorem~16 of \citet[Section~5.2]{KCG16}: it corresponds to the ``There exists an alternative
bandit problem'' part of the statement of the latter.} to Theorem~16 of \citet[Section~5.2]{KCG16}.
The corollary actually relies on a strong property of $\KL$ in this model:
not only is it symmetric, but it only depends on the expectation gaps between its arguments. Namely,
for all pairs $\cN(\mu,\sigma^2)$ and $\cN(\mu',\sigma^2)$ of distributions in
$\cD_{\sigma^2}$, for all $\Delta \in \R$,
\begin{equation}
\label{eq:KL-gap-based}
\KL\bigl( \cN(\mu,\sigma^2), \, \cN(\mu',\sigma^2) \bigr) = \frac{(\mu - \mu')^2}{2\sigma^2}
= \KL\bigl( \cN(\mu+\Delta,\sigma^2), \, \cN(\mu'+\Delta,\sigma^2) \bigr)\,.
\end{equation}
We introduce the following short-hand notation:
\[
C(\unu) \eqdef \sum_{a \neq a^\star(\unu)} \frac{2\sigma^2}{\Delta_a^2}\,.
\]

\begin{corollary}
\label{cor:16}
For all sequences of strategies and for all bandit problems $\unu$ in $\cD_{\sigma^2}$ with a unique optimal arm,
there exists a set of alternative bandit instances $(\unu^{(k)})_{k \neq a^\star(\unu)}$ in $\cD_{\sigma^2}$,
where each $\unu^{(k)}$ admits $k$ as a best arm and satisfies $C\bigl(\unu^{(k)}\bigr) \leq C(\unu)$,
and for which
\[
\frac{1}{T} \ln \max \biggl\{ \P_{\unu}\bigl(I_T \neq a^\star(\unu)\bigr),
\, \max_{k \ne a^\star(\unu)} \P_{\unu^{(k)}}\bigl(I_T \neq k\bigr) \biggr\}
\geq - 4 \, C(\unu)^{-1} - \frac{\ln 4}{T}\,.
\]
\end{corollary}

The proof provided below is highly specific to the Gaussian model
and exploits the gap-based rewriting~\eqref{eq:KL-gap-based} of the Kullback-Leibler divergence.
The calculations led would only extend to
models for which such gap-based rewritings of (upper and lower bounds on) the Kullback-Leibler divergence
would be available.

To compare the result of Corollary~\ref{cor:16} with the bound~\eqref{eq:LB-ABM10}
stemming from~\citet{ABM10}, note that
\[
C(\unu)^{-1}
\geq \frac{2}{\sigma^2} \min_{2 \leq k \leq K} \frac{\Delta_{(k)}^2}{k}\,.
\]

\begin{proof}
We apply Proposition~\ref{prop:LB-BH}
with the distributions $\zeta_k = \mathcal{N}(\mu^\star + \Delta_k, \sigma^2)$, for $k \neq a^\star(\unu)$.
On the one hand, the bound of Proposition~\ref{prop:LB-BH} involves
\[
\sum_{a \neq a^\star(\unu)} \frac{1}{\KL(\nu_a, \zeta_a)} = \sum_{a \neq a^\star(\unu)} \frac{2\sigma^2}{\bigl(\underbrace{\Ed(\nu_a)}_{\mu^\star-\Delta_a} - \underbrace{\Ed(\zeta_a)}_{\mu^\star+\Delta_a}\bigr)^2} = \frac{C(\unu)}{4} \,.
\]
On the other hand, for $k \neq a^\star(\unu)$, as the best arm of $\unu^{(k)}$ is $k$, with associated expectation $\mu^\star + \Delta_k$,
\begin{multline*}
C\bigl(\unu^{(k)}\bigr) = \sum_{a \neq k} \frac{2\sigma^2}{( \mu^\star + \Delta_k - \mu_a )^2} = \frac{2\sigma^2}{\Delta_k^2} + \sum_{a \notin \{k, a^\star(\unu)\}} \frac{2\sigma^2}{( \mu^\star + \Delta_k - \mu_a )^2} \\
	\leq \frac{2\sigma^2}{\Delta_k^2} + \sum_{a \notin \{k, a^\star(\unu)\}} \frac{2\sigma^2}{( \mu^\star - \mu_a )^2} = \sum_{a \neq a^\star(\unu)} \frac{2\sigma^2}{\Delta_a^2} = C(\unu) \,.
\end{multline*}
These two observations conclude the proof of Corollary~\ref{cor:16}.
\end{proof}

\paragraph*{Proof of Proposition~\ref{prop:LB-BH}.}
We conclude this section with a proof of Proposition~\ref{prop:LB-BH}.
It relies on the Bretagnolle-Huber inequality (\citealp{BH79}), which states that,
for all $p, q \in [0, 1]$,
\begin{equation}	
\label{eq:BH}
p + 1 - q \geq \frac{1}{2} \exp\Bigl(-\KL\bigl(\Ber(p), \Ber(q)\bigr) \Bigr) \,.
\end{equation}

\begin{proof}
We fix distributions $\zeta_k$ abiding by the conditions of the proposition
and also fix $T \geq 1$. We will prove below that, for all
convex weights $(u_b)_{b \ne a^\star(\unu)}$, i.e.,
non-negative weights summing up to~$1$,
\begin{equation}
\label{eq:proof-BHu}
\frac{1}{T} \ln \max \biggl\{ \P_{\unu}\bigl(I_T \neq a^\star(\unu)\bigr),
\, \max_{k \ne a^\star(\unu)} \P_{\unu^{(k)}}\bigl(I_T \neq k\bigr) \biggr\}
\geq  - \max_{b \neq a^\star(\unu)} \bigl\{u_b \, \KL(\nu_b, \zeta_b) \bigr\} - \frac{\ln 4}{T}\,,
\end{equation}
from which Proposition~\ref{prop:LB-BH} follows, by optimizing the obtained lower bound, i.e.,
by taking
\[
u_b = \Biggl( \sum_{a \neq a^\star(\unu)} \frac{1}{\KL(\nu_a, \zeta_a)} \Biggr)^{-1} \times \frac{1}{\KL(\nu_b, \zeta_b)} \,.
\]
We now fix convex weights $(u_b)_{b \neq a^\star(\unu)}$ and prove~\eqref{eq:proof-BHu}.
As $b \ne a^\star(\unu)$ and $b$ is the unique optimal arm of $\unu^{(b)}$,
for the first inequality, and by the Bretagnolle-Huber inequality~\eqref{eq:BH}, for the second inequality,
\begin{align*}
\P_{\unu} \bigl(I_T \neq a^\star(\unu) \bigr) + \P_{\unu^{(b)}}\bigl(I_T \neq b \bigr) &\geq \P_{\unu} \bigl(I_T \neq a^\star(\unu) \bigr) + \P_{\unu^{(b)}}\bigl(I_T = a^\star(\unu) \bigr) \\
	&\geq \frac{1}{2} \exp\Bigl(-\KL\bigl(\Ber(p_T), \Ber(q_T)\bigr) \Bigr) \,,
\end{align*}
where $p_T \eqdef \P_{\unu} \bigl(I_T \neq a^\star(\unu) \bigr)$ and $q_T \eqdef  \P_{\unu^{(b)}}\bigl(I_T \neq a^\star(\unu) \bigr)$.
Inequality~\eqref{eq:chainrule+datacompressing} reads, in the present case, as $\unu$ and $\unu^{(b)}$ only differ at arm $b$,
\begin{align*}
\KL\bigl(\Ber(p_T), \Ber(q_T)\bigr) \leq  \E_{\unu}\bigl[N_b(T)\bigr] \, \KL(\nu_b, \zeta_b)\,.
\end{align*}
Using $\max\{u, v\} \geq (u + v) / 2$ after collecting all bounds obtained so far yields
\[
\max \Bigl\{ \P_{\unu} \bigl(I_T \neq a^\star(\unu) \bigr) , \P_{\unu^{(b)}}\bigl(I_T \neq b \bigr)\Bigr\}
\geq \frac{1}{4} \exp\Bigl(- \E_{\unu}\bigl[N_b(T)\bigr] \, \KL(\nu_b, \zeta_b) \Bigr) \,.
\]
We take the maxima over $b \ne a^\star(\unu)$ in both sides,
apply logarithms, and conclude the proof of~\eqref{eq:proof-BHu} by showing that
\begin{equation}
\label{eq:proof-BHu-ccl}
\min_{b \neq a^\star(\unu)} \Bigl\{ \E_{\unu}\bigl[N_b(T)\bigr] \, \KL(\nu_b, \zeta_b) \Bigr\}
\leq \max_{b \neq a^\star(\unu)} \bigl\{u_b \, \KL(\nu_b, \zeta_b) \bigr\}\,.
\end{equation}
Indeed,
\[
\sum_{b \neq a^\star(\unu)} \frac{\E_{\unu}\bigl[N_b(T)\bigr]}{T} \leq 1 = \sum_{b \neq a^\star(\unu)} u_b \,,
\]
so that there exists $b^\star \neq a^\star(\unu)$ such that $\E_{\unu}\bigl[N_{b^\star}(T)\bigr]/T \leq u_{b^\star}$.
We then have
\[
\min_{b \neq a^\star(\unu)} \Bigl\{ \E_{\unu}\bigl[N_b(T)\bigr] \, \KL(\nu_b, \zeta_b) \Bigr\}
\leq u_{b^\star} \, \KL(\nu_{b^\star}, \zeta_{b^\star})
\leq \max_{b \neq a^\star(\unu)} \bigl\{u_b \, \KL(\nu_b, \zeta_b) \bigr\}\,,
\]
as desired in~\eqref{eq:proof-BHu-ccl}.
\end{proof}

\end{document}